\documentclass[journal]{IEEEtran}

\usepackage{cite}
\usepackage{amsmath,amssymb,amsfonts, multicol, mathtools}
\usepackage{graphicx}
\usepackage{textcomp}
\usepackage{makecell}
\usepackage{microtype}
\usepackage{caption}
\usepackage{subcaption}

\usepackage[export]{adjustbox}
\usepackage{svg}
\usepackage{gensymb}

\usepackage{amsthm}
\usepackage{physics}
\usepackage{booktabs}
\usepackage{multirow}
\usepackage{breqn}

\usepackage{algorithm}
\usepackage{algpseudocode}

\usepackage[justification=centering]{caption}
\usepackage{prettyref}
\usepackage{cleveref}
\newtheorem{remark}{Remark}
\newtheorem{theorem}{Theorem}
\renewenvironment{proof}{{\bfseries Proof.}}{\qed}
\usepackage{cuted}

\newrefformat{tab}{Table~\ref{#1}}
\newrefformat{fig}{Fig.~\ref{#1}}
\newrefformat{sec}{Section~\ref{#1}}
\newrefformat{eqn}{(\ref{#1})}

\algnewcommand{\Initialize}[1]{%
	\State \textbf{Initialize:}
	\Statex \hspace*{\algorithmicindent}\parbox[t]{.8\linewidth}{\raggedright #1}
}

\usepackage{lineno}

\modulolinenumbers[5]

\newcommand{\argmax}{\arg\max}

\newcommand{\infin}{\infty}

\def\BibTeX{{\rm B\kern-.05em{\sc i\kern-.025em b}\kern-.08em
		T\kern-.1667em\lower.7ex\hbox{E}\kern-.125emX}}

\ifCLASSOPTIONpeerreview
\else
\fi

\begin{document}
	
	\title{Cooperative guidance of multiple missiles: a hybrid co-evolutionary approach}
	
	\author{Xuejing Lan, Junda Chen, Zhijia Zhao,~\IEEEmembership{Member,~IEEE,} and Tao Zou
		
		\thanks{The author(s) received no financial support for the research, authorship, and/or publication of this article. (Xuejing Lan and Junda Chen contributed equally to this work)(Corresponding author: Zhijia Zhao)}
		\thanks{X. J. Lan, J. D. Chen, Z. J. Zhao, and T. Zou are with the School of Mechanical and Electrical Engineering, Guangzhou University, Guangzhou 510006, China (e-mail:  lanxj@gzhu.edu.cn; CJD@e.gzhu.edu.cn; zhjzhaoscut@163.com; tzou@gzhu.edu.cn).}
	}
	\maketitle
	
	\begin{abstract}
		Cooperative guidance of multiple missiles is a challenging task with rigorous constraints of time and space consensus, especially when attacking dynamic targets. In this paper, the cooperative guidance task is described as a distributed multi-objective cooperative optimization problem. To address the issues of non-stationarity and continuous control faced by cooperative guidance, the natural evolutionary strategy (NES) is improved along with an elitist adaptive learning technique to develop a novel natural co-evolutionary strategy (NCES). The gradients of the original evolutionary strategy are rescaled to reduce the estimation bias caused by the interaction between the multiple missiles. 
		A hybrid co-evolutionary cooperative guidance law (HCCGL) is then developed by integrating the highly scalable co-evolutionary strategy and the proportional guidance law, with detailed convergence proof provided. Finally, simulations demonstrated the effectiveness and superiority of this guidance law in solving cooperative guidance tasks with high accuracy, with potential applications in other multi-objective optimization, dynamic optimization, and distributed control scenarios. 

	\end{abstract}
	
	
	\section{Introduction}
	\label{sec:introduction}
	\IEEEPARstart{M}{odern} penetration of air defense systems of the target requires coordinated attacks with multiple missiles. However, the rapid development of detection technologies and close-in weapon systems (CIWS) has decreased the chances of successful impact with a single conventional missile.~\cite{jeon2010homing}. In addition to increasing the difficulty of interception, the cooperative guidance strategy of multiple missiles is also crucial to the lethal effect of the final impact.
	Usually, the cooperative guidance of multiple missiles belongs to the phase of terminal guidance, where accurate target information can be obtained with active radar systems or other detection devices. The existing cooperative guidance laws can be roughly divided into two categories. One is the analytical method to find closed-form solutions, which is mainly based on sliding mode control, optimal control, and multi-agent consensus theory. The other is the intelligent method which generally adopts heuristic intelligent optimization algorithm and reinforcement learning (RL) theory. 
	
	The analytical cooperative guidance method has been proven to be robust and efficient for practical application~\cite{ma2013guidance,xiong2018hyperbolic,liRobustCooperativeGuidance2019,he2021computational,ratnoo2008impact}. Based on fundamental proportional navigation (PN), Jeon et. al developed cooperative proportional navigation (CPN) where the on-board time-to-go of the missile is used as the navigation gain~\cite{jeon2010homing}. It is a simple but effective approach for achieving time consensus. Ma developed a composite guidance law, which can be decomposed into the direction along the line of sight (LOS) and the direction perpendicular to LOS~\cite{ma2013guidance}, corresponding to time and space cooperative respectively. Furthermore, time cooperative control is achieved with the combination of PNG and impact time error feedback~\cite{jeon2006impact}, where the undirected topology is adopted to establish communication relationships. Based on the optimal control approach, a variant of the hyperbolic tangent function is proposed in \cite{xiong2018hyperbolic} to force early control of velocity and impact angle.

	However, with the increasing demand for developing high-precision weapon systems, intelligent cooperative guidance method is increasingly regarded as a necessary auxiliary option. 
	In recent years, the reinforcement learning theory has attracted much attention because of its ability to learn online based on environmental feedback~\cite{gaudet2020reinforcement,la2014multirobot,dong2021composite,kong2020maneuver, albaba2021driver, chen2021communication}.
	According to the training structures, existing reinforcement learning algorithms for multi-agent systems can be roughly divided into four types, which are \emph{Fully decentralized training, decentralized execution;   Fully centralized training, decentralized execution; Centralized training, centralized execution, and value decomposition methods}.
	Some of these algorithms have achieved satisfactory results in coping with problems with low complexity and accuracy requirements.
	In \cite{dong2020reinforcement}, \cite{he2021computational}, and  \cite{Liangchen2021Metalearning}, the state-of-the-art reinforcement learning frameworks have demonstrated their effectiveness in the guidance task. Zhang et.al proposed a gradient-descent-based reinforcement learning method in the actor-critic framework and achieved consensus control for multi-agent systems by following a tracking leader~\cite{zhang2016data}.  
	But the two challenges of \emph{Nonstaionarity}  and \emph{Partial Observability} \cite{nguyen2020deep} will lead to saturated output or coordination loss of multi-agent systems, which greatly reduces the accuracy of the value function. In addition, the use of value function in reinforcement learning is not suitable for continuous control tasks with large search spaces. Thus, these limitations of RL impede the development of reinforcement learning in cooperative guidance. 
	
	It is an excellent way to solve the above problems by removing the value function of reinforcement learning and optimizing in solution space with evolutionary strategy (ES), which is more robust and invariant to real-time rewards because it optimizes towards the objective function directly~\cite{brockhoff2010mirrored}. Moreover, as described in \cite{salimans2017evolution}, ES is tolerant of long horizontal and implicit solutions, which is exactly consistent with the need for cooperative guidance. 
	The natural evolutionary strategy (NES) is the latest branch of ES, and shows good performance in solving high-dimensional continuous multimodal optimization problems, by using the natural gradient information estimated according to the fitness expectation of the population~\cite{brockhoff2010mirrored,wierstra2014natural,salimans2017evolution}. 
	Similar algorithms named co-evolutionary algorithm have been discussed in \cite{xu2017environment} and \cite{qu2013improved}, which focus on solving multi-objective optimization problems by dividing the overall objective into sub-objectives, such to optimize and evaluate together. Another idea is to evolve multiple populations for the same goal, and manually regulate the constraints of each population for faster convergence or fuller exploration~\cite{qu2013improved}. As represented in \cite{qu2013improved}, the concept of co-evolution refers to multi-threads of training processes. Note that these methods do not use the natural gradient information as in NES, and the non-stationary issue discussed above is not considered. 
	
	When optimizing in continuous parameter(solution) space, it is very important to apply adaptive technology. While a learning rate adaption method based on the quality of gradients is often not easy to estimate, a simple workaround would be leveraging the shifting distance of parameters to adapt the learning rate. As shown in \cite{wang2022instance}, the size of population was adjusted depending on the novelty metric and quantity metric, which reflected the complexity of the dynamic environment. The estimation of distribution algorithm (EDA) was applied to continuous control by searching the optimal parameter distribution~\cite{larranaga2001estimation,karshenas2013regularized}.  A variety of evolutionary methods were investigated to design the multi-objective missile guidance law~\cite{omar2011multiobjective}. Maheswaranathan proposed a surrogate gradient to reduce the evaluation costs~\cite{maheswaranathan2019guided}. These works reveal the enormous potential of searching in parameter space, rather than directly searching in parameter space.

	Therefore, an NES-based co-evolutionary algorithm naming as the natural co-evolutionary strategy (NCES) is developed in this paper to distress the dilemma faced by RL in the cooperative guidance task. Considering the advantages of searching in parameter space, the co-evolutionary algorithm is improved in this work by rescaling the gradient information to reduce the estimation bias introduced by neighboring populations. As discussed in \cite{del2019bio, gray2018multiagent}, most of today's bio-inspired algorithm innovations are based on experimental observation rather than meticulous theoretical support. Whereas in this work, we try to dig into the depths of complex optimization and provide proof as sensible as possible through the presentation of graphs and deduction. 
	Via integrating the NCES algorithm, a hybrid co-evolutionary cooperative guidance law (HCCGL) is further developed to solve the challenging missile guidance problem. Extensive empirical results on various engagement scenarios verified the effectiveness of the proposed guidance law.
	The main contributions of this work are summarized as follows:
	\begin{enumerate}[]
		\item To address the issues of non-stationarity and continuous control faced by cooperative guidance, an NCES algorithm is formulated and incorporated into a novel guidance law as an alternative to RL in the cooperative guidance task.
		
		\item The rigorous constraints of time and space consensus in cooperative guidance are integrated and designed as the fitness function for each missile. An MLP-based policy network is constructed and learned to optimize the fitness function.
		
		\item The proposed HCCGL  has advantages in achieving high precision for cooperative guidance tasks, even with dynamic targets and random initial conditions.
		
	\end{enumerate}
	
	The rest of the paper is organized as follows. The problem formulation is elaborated in \prettyref{sec:problemFormulation}, and the proposed cooperative guidance law is discussed in \prettyref{sec:naturalCo-evolutionaryStrategy}. In \prettyref{sec:simulations}, experiments under various configurations are implemented. Finally, conclusions are made in \prettyref{sec:conclusions}.

	\section{Problem Formulation}
	\label{sec:problemFormulation}
	
	\subsection{Engagement geometry}
	The two-dimensional engagement geometry between multiple missiles and one target is shown in \prettyref{fig:engagement1}, where the inertial coordinate frame $OXY$ represents the horizontal plane.
	There are $n$ missiles in total. The index $M_i$ denotes the $i^{th}$ missile, and $T$ represents the target. $ V_{mi}, \Xi_{mi}, \alpha_{mi}$, and $\delta_{mi}$ represent the velocity, line of sight (LOS) angle, flight-path angle, and heading angle of the $i^{th}$ missile, respectively. $a_{li}$ and $a_{vi}$ represent the lateral acceleration and the thrust acceleration to be designed for the $i^{th}$ missile, which are perpendicular to and align with the direction of $V_{mi}$, respectively. $V_T, \Xi_{T}, \alpha_{T}$, and $\delta_{T}$ are the velocity, LOS angle, flight-path angle, and heading angle of the target, respectively. The lateral acceleration of the target is denoted by $a_T$. 
	\begin{figure}[htbp]
		\centering
		\adjincludegraphics[scale=0.5,trim={0.1\width, 0, 0, 0}, clip]{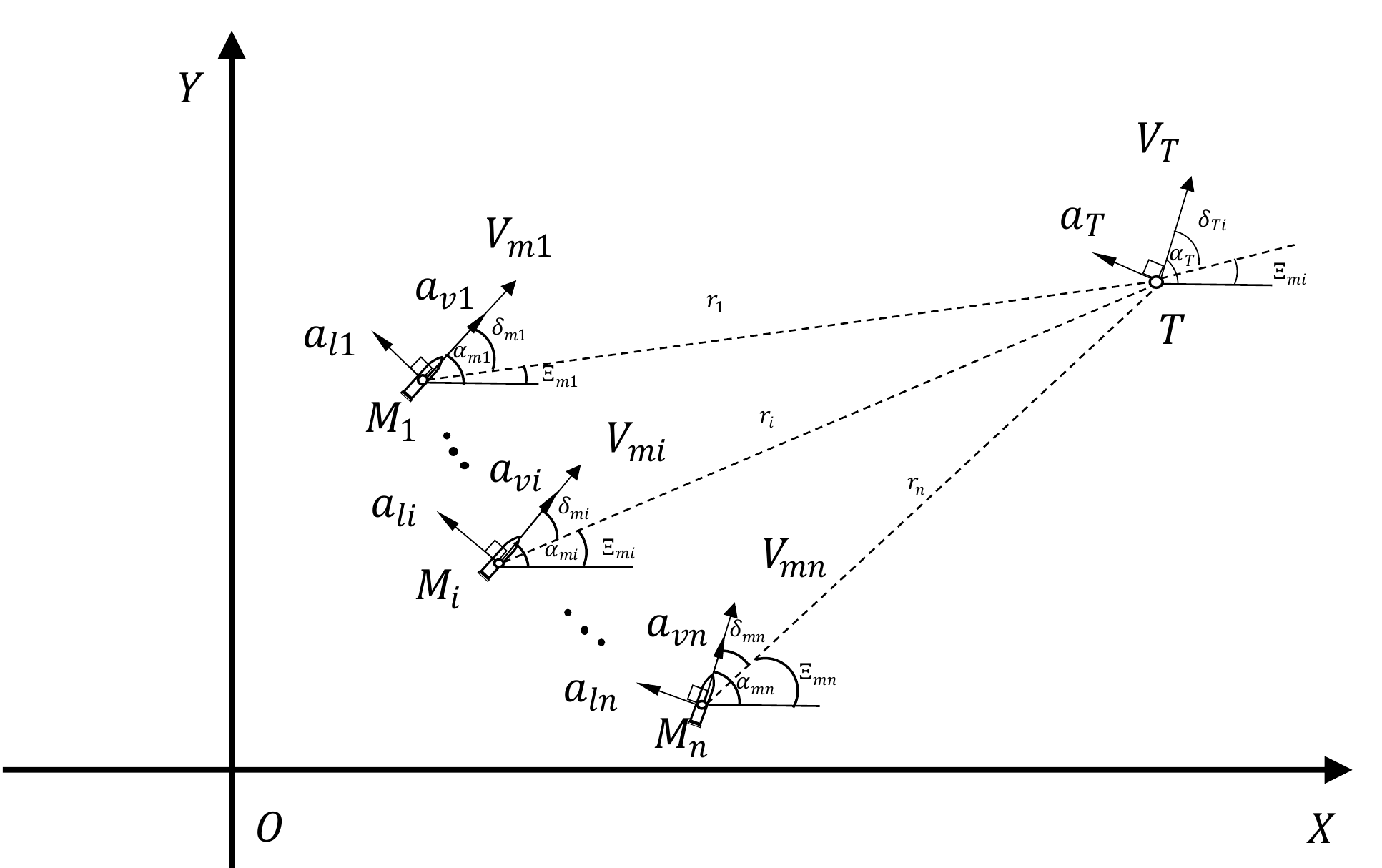}
		\caption{Two-dimensional engagement geometry}
		\label{fig:engagement1}
	\end{figure}
	
	The dynamic equations of the $i^{th}$ missile and the target are as follows:
	\begin{equation}\label{eqn:nonstatsystemdynamic}
		\begin{cases}
			\dot{r_i} = -V_{mi}\text{cos}\,\delta_{mi}+V_{T}\text{cos}\,\delta_{Ti}\\
			r_i\dot{\Xi}_{mi}=-V_{mi}\text{sin}\,\delta_{mi}+V_{T}\text{sin}\,\delta_{Ti}\\
			\dot{\alpha_{mi}}=a_{li}/V_{mi}\\
			\dot{\alpha_T}=a_T/V_T\\
			\dot{V}_{mi}=a_{vi}\\
			\delta_{mi}= \alpha_{mi}-\Xi_{mi}\\
			\delta_{Ti}=\alpha_{T}-\Xi_{mi}
		\end{cases},
	\end{equation}
	where, $r_i$ represents the relative range between the $i^{th}$ missile and  the target. 
	The time-to-go of the $i^{th}$ missile $t_{go}^i$ refers to the time left from the current time until the interception:
	\begin{equation}
		\label{eqn:ttg}	
		t_{go}^i=-\frac{r_i}{\dot{r_i}},
	\end{equation}

	\subsection{Communication Topology}
	The communication relationship of the multiple missiles is depicted by a topology, where a set of nodes  $\mathcal{V} = \{v_1, v_2,...,v_n \}$ represents the $n$ missiles. The communications are represented by a set of edges $\xi \subseteq \mathcal{V}\times \mathcal{V}$ with an adjacency matrix $A=[a_{ij}]\in \mathbb{R}^{n\times n}$, where $a_{ij}=1$  if missile $j$ is able to communicate directly with missile $i$ , otherwise $a_{ij}=0$. $\mathcal{N}_i=\{j\in\mathcal{V}:(i,j)\in\epsilon\}$ is the set of neighboring missiles of the $i^{th}$ missile. In practical engineering, the communication topology is determined through comprehensive considerations of the communication cost and actual demand.
	In this work, the undirected topology shown in~\prettyref{fig:topologies} is adopted, enabling neighboring missiles to share information. 
	\begin{figure}
		\centering
		\includegraphics[width=0.2\textwidth]{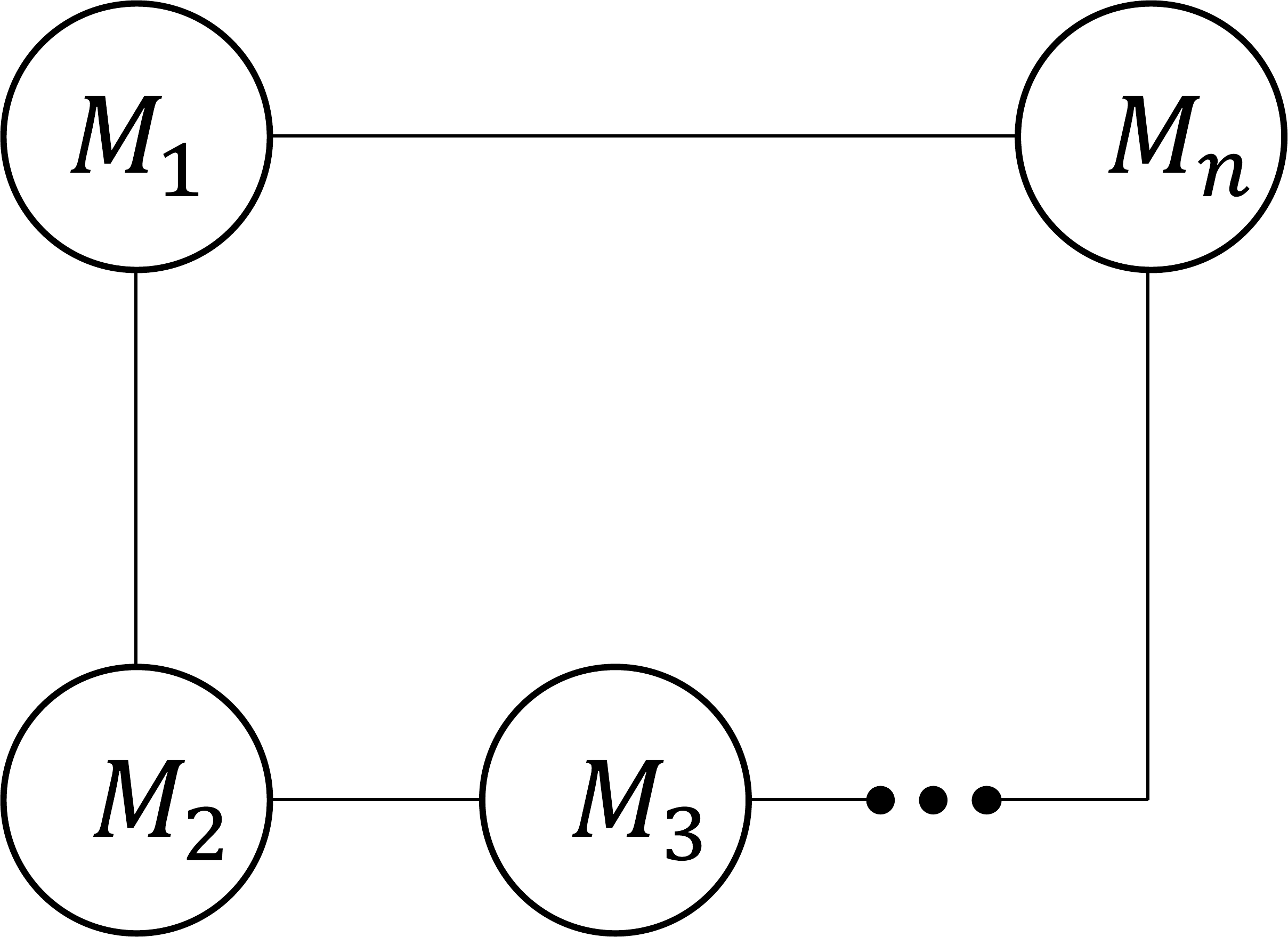}
		\caption{Communication topology}\label{fig:topologies}
	\end{figure}

	\subsection{Observation}
	For the multi-missile system, the complete observation information of the entire system is not available to each agent. Thus, the cooperative guidance problem is a partially observable Markov decision process (POMDP) described by
	\begin{equation}\label{eqn:mdptransition}
		O_i \times A_i \to O_i',
	\end{equation}
	where,  $O_i$ and $A_i$ represent the observation and action of the $i^{th}$ missile. $O_i'$ is the observation of the $i^{th}$ missile at next time step.
	
	The full state information of each missile consists of three components: personal features, target features, and error features shown in Table \ref{tab:observation}. $P_{mi}$ and $P_T $ represent the positions of the missile $i$ and the target in two-dimensional coordinates. The target features are estimated or detected through onboard equipment, and the estimation error is assumed to be negligible compared with the required guidance precision. The acquisition of accurate location information requires the support of powerful global positioning systems, here we only need relative error information. $e_t^i $ is the consensus error of time of the missile $i$:
	\begin{equation}\label{eqn:errordefin}
		e_t^i = \sum_{j\in \mathcal{N}_i} (t_{go}^i - t_{go}^j),
	\end{equation}
	The consensus error of LOS angle of the missile $i$ is defined as:
	\begin{equation}\label{eqn:errordefin}
		e_a^i = \sum_{j\in \mathcal{N}_i} (e_\Xi^i - e_\Xi^j),
	\end{equation}
	where, $e_\Xi^i= \Xi_{mi}-\Xi_{di}$ is the LOS angle error of the missile $i$, and $\Xi_{di}$ is the desired impact angle of missile $i$:
	\begin{equation} \label{eqn:theta_di}
		\Xi_{di}=\Xi_{d1} + \sum_{j=1}^{i-1}\delta_{d}^j,
	\end{equation}
	where, $\delta_{d}^i$ is the desired relative impact angle between two missiles, and $\Xi_{d1}$ is the nominal desired impact angle of the first missile which is determined online.
	To increase the flexibility and autonomy of the intelligent missile system, the desired $\Xi_{di}$ can be adjusted adaptably instead of being a fixed value.
	
	\begin{table}[ht]
		\begin{center}
			\caption{Full state information of each missile}
			\label{tab:observation}
			\begin{tabular}{ll}
				\hline\hline
				Features                            & Symbols             \\ \midrule
				\multirow{4}{*}{Personal Features} & $P_{mi}=(x_i, y_i)$ \\
				& $\alpha_{mi}$       \\
				& $\Xi_{mi}$    \\
				& $V_{mi}$            \\ \midrule
				\multirow{3}{*}{Target Features}    & $P_T = (x_T, y_T)$  \\
				& $V_T$               \\
				& $\alpha_T$          \\ \midrule
				\multirow{2}{*}{Error Features}     & $e^i_t$               \\
				& $e^i_a$               \\
				& $e_\Xi^i$               \\ \hline\hline
			\end{tabular}
		\end{center}
	\end{table}
	

	\subsection{Fitness evaluation}
	
	The reward of each missile at one evaluation step consists of a terminal reward and a flight reward. The objective of the cooperative guidance task is to minimize the error  $e_{t}^i$, $e_a^i $, and $e_\Xi^i$. 
	Then, the terminal reward is defined as:
	\begin{equation} 
		r_{T}^i=(\gamma_a\cdot e^{-\xi_a |e_\Xi^i|} + \gamma_t\cdot e^{-\xi_t |e_t^i |})\cdot \epsilon(k),
	\end{equation}
	where, $\xi_a$, $\xi_t$, $\gamma_a$, $\gamma_t$ are constant coefficients. $\epsilon(k)$ is the step function defined as
	\begin{equation} \label{eqn:stepfunc}
		\epsilon(k)= \begin{cases}
			1, &\text{if $k$ is terminal step}\\
			0, &\text{otherwise} 
		\end{cases}.
	\end{equation}
	Thus, the terminal reward only reflects the results at the terminal step, and $r_{T}^i = \gamma_a + \gamma_t$ if and only if $e_\Xi^i=0$ and $e_t^i=0$.
	The flight reward is defined as:
	\begin{equation} 
		r_{F}^i=\beta_a(-1+e^{-k_a|e_a^i|})+ \beta_t(-1+e^{-k_t|e_t^i|}),
	\end{equation}
	where, $k_a, k_t$, $\beta_a$, and $\beta_t$ are positive constant coefficients.
	It can be inferred that $r_{F}^i \le 0$ is always true. $r_{F}^i = 0$ if and only if $e_a^i=0$ and $e_t^i=0$.
	Then, the fitness function of missile $i$ for the cooperative guidance task is defined by
	\begin{equation} 
		\begin{aligned}
			F_i &= \int_t (r_{F}^i+r_{T}^i)  \text{d}t.
		\end{aligned}
	\end{equation}
	Thus, the objective of the cooperative guidance task can be achieved by maximizing the fitness function of each missile.
	
	
	\subsection{Design of the cooperative guidance law}
	Based on the requirements of the cooperative guidance task, the guidance law proposed in this paper includes two parts: the tracking control part and the consensus control part. 
	The tracking control part is obtained by proportional navigation guidance(PNG) :  
	\begin{equation} \label{eqn: trackingcontroller}
		u_{pi} = [\beta   \dot{\Xi}_{mi}V_{mi},0]^T,
	\end{equation}
	where, $\beta  $ is the navigation constant. Note that the tracking control part only designs the lateral acceleration.
	
	The consensus control part is modeled by a neural network expressed as
	\begin{equation} \label{eqn:forward1}
		\begin{aligned}
			u_{ei} &=W_{3i}^T \cdot \psi(Z_{2i}),\\
			Z_{2i} &= W_{2i}^T\cdot \phi(Z_{1i}),\\
			Z_{1i} &= W_{1i}^T\cdot \phi(X_i),
		\end{aligned}
	\end{equation}
	where $W_{3i}  \in \mathbb{R}^{q_2\times 2}$, $W_{2i}  \in \mathbb{R}^{q_1\times q_2}$, and $W_{1i}  \in \mathbb{R}^{3 \times q_1}$ denote the weight matrices of the output layer. $Z_{1i}$ and $Z_{2i}$ are the outputs of the first and second hidden layers. $q_1$  and $q_2$ are the numbers of neurons in each layer. $\psi(\cdot)$  is the bounded activation function $Tanh(\cdot)$ with $||\psi(\cdot)|| \leqslant a_{lmax}$, and $\phi(\cdot)$ is the common activation function $Sigmoid(.)$.
	The input state vector $X_i$ is selected as: 
	\begin{equation} \label{eqn:forward1}
		X_i=[ e_a^i, e_t^i, e_\Xi^i ]^T.
	\end{equation}
	Thus, the guidance law of the missile $i$ is presented as:
	\begin{equation} \label{eqn:composedcommand}
		u_i = (1-\eta)u_{pi} + \eta u_{ei},
	\end{equation}
	where $\eta$ is the guidance gain trading off the tracking control part and the consensus control part.

	\section{Natural Co-evolutionary Strategy }
	\label{sec:naturalCo-evolutionaryStrategy}
	
	\subsection{bottleneck of RL}
	Reinforcement learning is a generic term for a class of value-oriented algorithms. It focuses on solving problems of Markov Decision Process (MDP), which also apply to the guidance problem.
	
	Assume there are $n$ agents in the environment. The joint action set is denoted by $A=A_1\times A_2\times... \times A_n$, and the system state is denoted by $S$. At each timestep, each agent takes one step, and with a certain probability, the transition occurs. This can be represented as $S\times A \to S'$, where $S'$ is the next system state after the transition. The reward is given as $S \times A\times S' \to \mathbb{R}^n$.In a deterministic environment, the transition probability is 1. For the multi-agent system, the $S$ can be decomposed into individual observations: $S=O_1+O_2+...+O_n$. A sufficient set of observations must be capable of representing the complete system state. In most cases, the agent does not have access to the complete information of the system. This means they only get partial observations instead of the complete state, making the problem a Partial Observable Markov Decision Process (POMDP).
	
	Two challenges, \emph{Nonstationarity} and \emph{partial observability} \cite{nguyen2020deep}, impede the research for multi-agent systems. Tons of algorithms have popped up focusing on solving this kind of problem. According to the training process, we can roughly divide them into four types:
	\begin{itemize}
		\item  Fully decentralized training, decentralized execution;
		\item  Fully centralized training, decentralized execution;
		\item Centralized training, centralized execution;
		\item Value decomposition methods.
	\end{itemize}
	The existing works have achieved satisfactory results in coping with problems of less complexity and less requirement for precise control. It has been well-investigated that an ill-distributed value function would seriously stagnate performance. Exploration technologies, such as \emph{Ornstein-Uhlenbeck} noise and stochastic exploration, are used to alleviate this problem.

	Applying reinforcement learning (RL) theory, it is possible to handle control tasks with either a single missile \cite{gaudetReinforcementLearningAngleOnly2020, han2002state}, or discrete action space \cite{li2020adaptive,kong2020maneuver}. However, for control tasks with multiple agents(missiles), inefficient exploration and non-stationarity can lead to a deterioration of the accuracy of the value function, resulting in either saturated control or coordination loss.
	Value functions can be advantageous for discrete control, but can be flawed for continuous control tasks with large search spaces. Approaches that constrain the policy space have been discussed in \cite{thananjeyan2021recovery}, but they heavily rely on prior knowledge and do not scale well to different scenarios. As an alternative, evolutionary strategies have abandoned the use of value functions and have shown the outstanding capability for the aforementioned issues.

	\subsection{Natural evolutionary Strategy in multi-agent POMDP}
	
	In the evolutionary strategy, individual agent (or its policy) is expressed as a \emph{population}, the group of populations and the environment constitute the \emph{ecosystem}. The objective is to develop the optimal strategy for the group of populations to maximize the \emph{fitnesses} of the ecosystem. For cooperative tasks, the optimal strategy of the ecosystem will be exactly the optimal policy for each population.
	\begin{equation}\label{eqn:optimalpolicy}
		\argmax_{u^*_{tot}}F_{tot}(u^*_{tot}) = \begin{pmatrix}
			\argmax_{u_1}F_{1}(u_1) \\
			\;\;\vdots \\
			\argmax_{u_n}F_{i}(u_n)
		\end{pmatrix},
	\end{equation}
	where, $u_i$ which is defined in \prettyref{eqn:composedcommand} represents the policy of the $i$th population and $u_{tot}^*=[u_{i}^*]_{i=1}^{i=n}$ is the joint matrix of individual optimal policy. $F_i(\cdot)$ is its corresponding fitness function and $F_{tot}(\cdot)$ is the joint policy fitness function, more details can be viewed in \cite{sonQTRANLearningFactorize2019}. However, the inverse is not true:
	
	\begin{equation}\label{eqn:verseoptimalpolicy}
		\begin{pmatrix}
			\argmax_{u_1}F_{1}(u_1) \\
			\;\;\vdots  \\
			\argmax_{u_n}F_{i}(u_n)
		\end{pmatrix}\neq\argmax_{u^*_{tot}}F_{tot}(u^*_{tot}).
	\end{equation}
	This is because the optimal fitness obtained by one population may be based on the suboptimal fitness obtained by other populations. When the other populations evolve, the previous optima is easy to be broken. To overcome this nonstationary issue,  it is best for all populations to evolve simultaneously, that is co-evolution. Each generation updates its parameter at the same time, instead of updating sequentially, mapping in slight variance in fitness values.
	
	
	\subsection{Optimization in co-evolutionary parameter space}\label{sec:co-evoAnalysis}
	The gradient information is obtained by measuring the contribution of each sample. The parameters of the population are defined as $\theta$, and $\theta'$ represents that of the next generation. $p_{\psi}(\theta'|\theta)$ is the distribution function of $\theta'$ under $\theta$, where $\psi$  is the intrinsic parameter. Then the expectation fitness of the next generation is expressed as:
	\begin{equation}\label{eqn:estimationoffitness}
		\mathbb{E}_{\theta' \sim p_\psi(\theta'|\theta)}F(\theta')=\int_{\theta'} p_\psi(\theta'|\theta)F(\theta') \mathrm{d}\theta'.
	\end{equation}
	
	The derivative of Eq.~(\ref{eqn:estimationoffitness}) with respect to $\theta$ is 
	\begin{equation}\label{eqn:gradient1}
		\triangledown_{\theta}\mathbb{E}_{\theta' \sim p_\psi(\theta'|\theta)}F(\theta')=\mathbb{E}_{\theta'}\{ \triangledown_{\theta'}\text{log}\,p_\psi(\theta' | \theta)F(\theta')\}. 
	\end{equation}
	
	If we represent $\theta'$ as $\theta + \epsilon$, then we have the similar equation

	\begin{equation} \label{eqn:gradient2}
		\triangledown_{\theta}\mathbb{E}_{\epsilon \sim p_\psi(\epsilon)}F(\theta + \epsilon)=\mathbb{E}_{\epsilon}\{ \triangledown_{\epsilon}\text{log}\,p_\psi(\epsilon)F(\theta+\epsilon)\} .
	\end{equation}
	
	In an ecosystem with multiple populations, populations will interact and affect the evolutionary process. Thus, the fitness function of the  $i$th population is represented by $F_i( \varsigma _i)$, where $\varsigma _i=  \{\theta_i, \theta_j : j\in \mathcal{N}_i\}$ represents the parameter set of the $i$th population and its neighboring populations.
	The expected joint fitness of the next generation is expressed as:
	\begin{equation}\label{eqn:gradientj}
		\mathbb{E}\{F_i(\varsigma _i')\}=\int_{\varsigma_i'} p_\psi(\varsigma_i '|\varsigma_i )F_i(\varsigma_i ') \mathrm{d}\varsigma_i '.
	\end{equation}
	where, $p(\varsigma_i '|\varsigma_i )$ is the joint probability distribution of the next generation over $\varsigma_i$. Assume that $\theta_i'$ and $\theta_j'$ are sampled independently, we have $p(\varsigma_i '|\varsigma_i )= p(\theta_i')\prod_{j\in \mathcal{N}_i} p(\theta_j')$.

	The gradient of the joint fitness with respect to $\theta_i$ is expressed as
	
	\begin{equation}\label{specificgradient}
		\begin{aligned}
			\triangledown_{\theta_i} \mathbb{E} \{ F_i(\varsigma_i') \}
			&= \triangledown_{\theta_i}\int_{\varsigma_i'}p(\varsigma_i')F_i(\varsigma_i') {\rm d}\varsigma_i',\\
			&=  \int_{\theta_i'}\int_{\theta_j'}\cdot\cdot\cdot\triangledown_{\theta_i} p(\theta_i')F_i(\varsigma_i') \prod_{j\in \mathcal{N}_i} \\[\jot] 
			&\quad  p( \theta_j'){\rm d}\theta_i'\prod_{j\in \mathcal{N}_i}{\rm d}\theta_j'\\
			&= \int_{\theta_j'}\cdot\cdot\cdot\int_{\theta_i'}\left[\triangledown_{\theta_i}{logp(\theta_i')}F_i(\varsigma_i')\right] p(\theta_i')\prod_{j\in \mathcal{N}_i} \\[\jot]
			&\quad p( \theta_j'){\rm d}\theta_i'\prod_{j\in \mathcal{N}_i}{\rm d}\theta_j'\\
			&=
			\mathbb{E}_{(\varsigma_i')}\{\triangledown_{\theta_i} {logp(\theta_i')}F_i(\varsigma_i')\}.
		\end{aligned}
	\end{equation}
	
	Note that it has the same format as the version of a single population, it seems to be fine if we just keep the original equation. The influence of $\theta_c'$ is counteracted through the calculation of its expectation. However, it is known that the expectation of the joint distribution is approximated through sampling with a limited size. Although individuals are sampled without bias(unbiased estimation), there exists an intrinsic bias for inadequate sampling, and the bias will grow linearly with an increment of distribution dimensionality. So it can be a serious issue when taking the expectation of all neighboring parameters, and the sample size stays relatively small.
	
	It is not necessary to take account of all parameters since only the expectation of $\theta_i$ is actually needed. To alleviate the incremental bias, we propose to approximate only the expectation of the parameter of the current population $\theta_i$ and ignore its neighbor parameters, which is
	\begin{equation}\label{eqn:expofoneself}
		\mathbb{E}_{\theta_i'}F_i(\theta_i') = \int_{\theta_i'}F_i(\theta_i')p(\theta_i')\mathrm{d}\theta_i'.
	\end{equation} 
	Though $p(\theta_i')$ is available for independent distribution, it is infeasible to obtain $F_i(\theta_i')$, since all agents are sampled and evaluated together. However, the expectation of individual fitness can be approximated by the multiplication between the original fitness $F_i(\varsigma_i')$ and
	its confidence. The rectified expectation is expressed as
	\begin{equation}\label{eqn:recifiedexpofoneself}
		\mathbb{E}_{\theta_i'}F_i(\theta_i') = \int_{\theta_i'}F_i(\varsigma_i')p(\theta_i')\prod_{c\in \mathcal{N}_i} p( \theta_c')\mathrm{d}\theta_i',
	\end{equation}
	where, $\prod_{c\in \mathcal{N}_i}p( \theta_c')$ is the confidence, and $\theta_c'$  represents the samples that appear along with $\theta_i'$. In this way, the bias of estimating the expectation of the neighboring distributions is addressed. The gradient after modification is
	
	\begin{equation}\label{eqn:gradient5}
		\begin{split}
			\triangledown_{\theta_i} \mathbb{E}_{\theta_i' } F_i(\theta_i') = \int_{\theta_i'}\left[\triangledown_{\theta_i}{logp(\theta_i')}F_i(\varsigma_i')\right] p(\theta_i')\prod_{c\in \mathcal{N}_i} p(\theta_c'){\rm d}\theta_i'\\
			= \mathbb{E}_{\theta'_i}\left\{\triangledown_{\theta_i} {logp(\theta_i')}F_i(\varsigma_i')\prod_{c\in \mathcal{N}_i}p(\theta_c')\right\}.
		\end{split}
	\end{equation}
	
	\begin{remark}
		We refer to strategies that use \prettyref{eqn:gradient5} as the updating policy as the natural co-evolutionary strategy (NCES). The natural co-evolutionary strategy is a strategy for evolutionary algorithms that updates the weights of the population in multi-agent systems. Compared to other existing evolutionary strategy algorithms, NCES alleviates the incremental bias caused by neighboring parameters and achieves better performance in cooperative continuous control tasks.
	\end{remark}
	
	The core idea is that although the individual fitness $F_i(\theta_i')$ does not exist, the expectation of the individual fitness does, and is invariant to the parameter distributions of its neighboring agents, so the expectation of the individual agent's fitness should be calculated instead of including the expectation of neighboring agents. Let's denote the expectation of the objective function over $(\varsigma_i')$, which is $\triangledown_{\theta_i} {logp(\theta_i')}F_i(\varsigma_i')$, by $\phi(\varsigma_i')$ and the expectation of the objective function over $\theta_i'$ by $\phi(\theta_i')$, such that 
	\begin{equation}\label{eqn:estoversingleparameter}
		\phi(\theta_i')=\int_{\theta_c'}\phi(\varsigma_i')p(\theta_c')\text{d}\theta_c'.
	\end{equation} 
	
	To visualize the sampling estimation process, we use a variant of \emph{eggholder} as the objective function for demonstration, which is defined in \prettyref{eqn:eggholder}, since the real objective is too expensive to obtain. 
	
	\begin{table*}
		\caption{A variant of the eggholder function}
		\centering
		\begin{minipage}{0.75\textwidth}
			\begin{align}\label{eqn:eggholder}
				\phi(\varsigma_i') 
				&= \frac{(30\theta_{c}'+47)\sin\sqrt{|30\theta_{c}'+15\theta_{i}'+47|}-30\theta_{i}'\sin\sqrt{|30\theta_{i}'-(30\theta_{c}'+47)|}}{200} - 0.2
			\end{align}	
			\medskip
			\hrule
		\end{minipage}
	\end{table*}
	
	Assume there exists one neighboring population $\theta_c$ for $\theta_i$ with the size of 400, the sampled individuals are shown in \prettyref{fig:estgrad3d}, following a bivariate normal distribution, and the parameter spaces are confined to $\theta_{i}'\in[-2,2], \theta_{c}'\in[-2,2]$. 
	Since $\theta_i'$ and $\theta_c'$ are sampled independently, the individuals can be considered to be sampled from $p(\theta_i')$ only, which is represented by the sample points in \prettyref{fig:estgrad2d}. In this objective graph with single-dimensional parameter space, the real objective curve expressed in solid line is obtained by \prettyref{eqn:estoversingleparameter}. In order to standardize the scope, all the sampled data including the real objective values are uniformly scaled to the range [0,1], and such standardization does not affect the directionality of the estimated gradient.
	
	\begin{figure}[htbp]
		\centering
		\adjincludegraphics[height=0.5\linewidth]{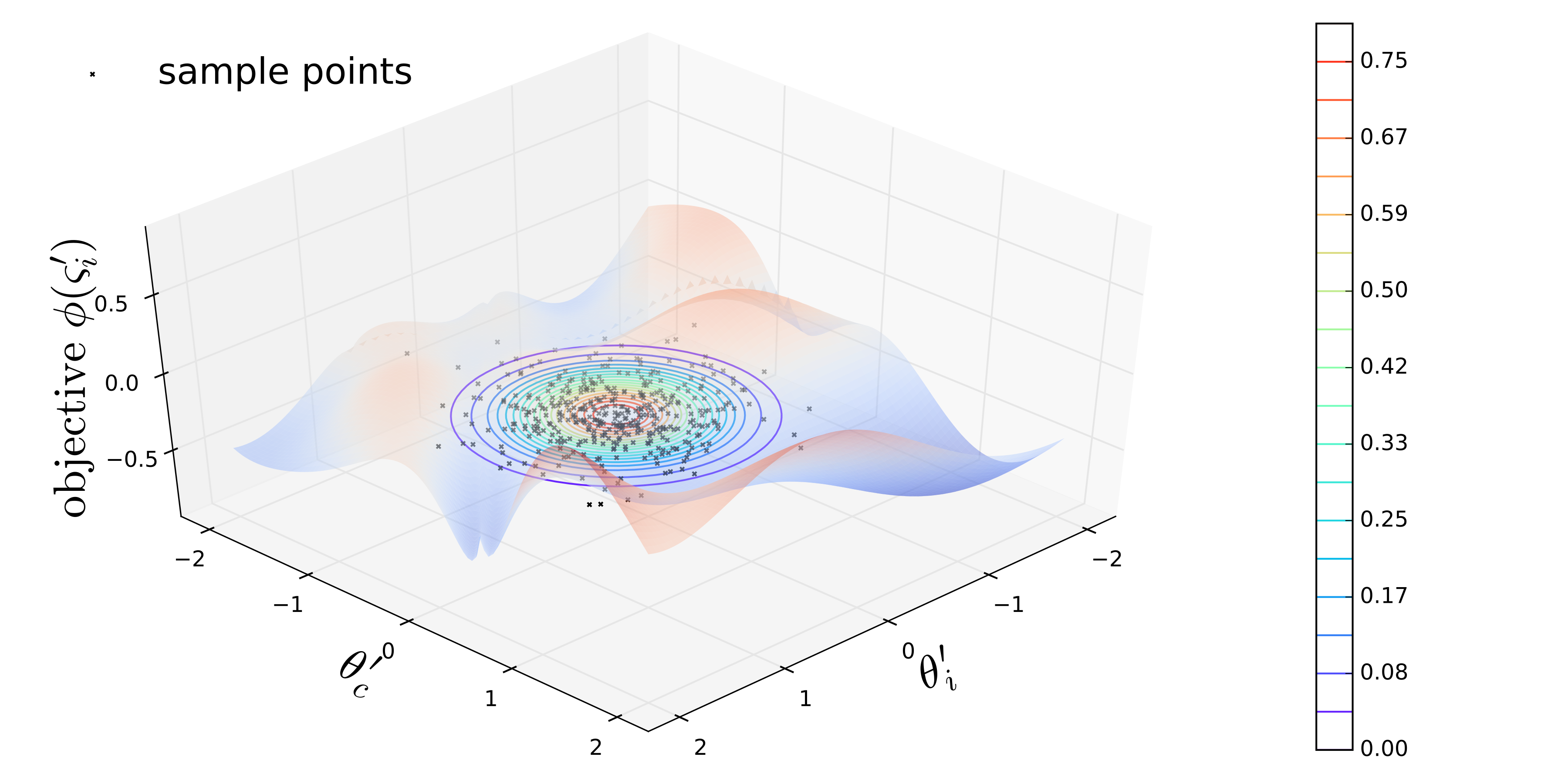}
		\caption{Estimate gradient through sampling}\label{fig:estgrad3d}
	\end{figure}
	\begin{figure}[htbp]
		\centering
		\adjincludegraphics[height=0.4\linewidth]{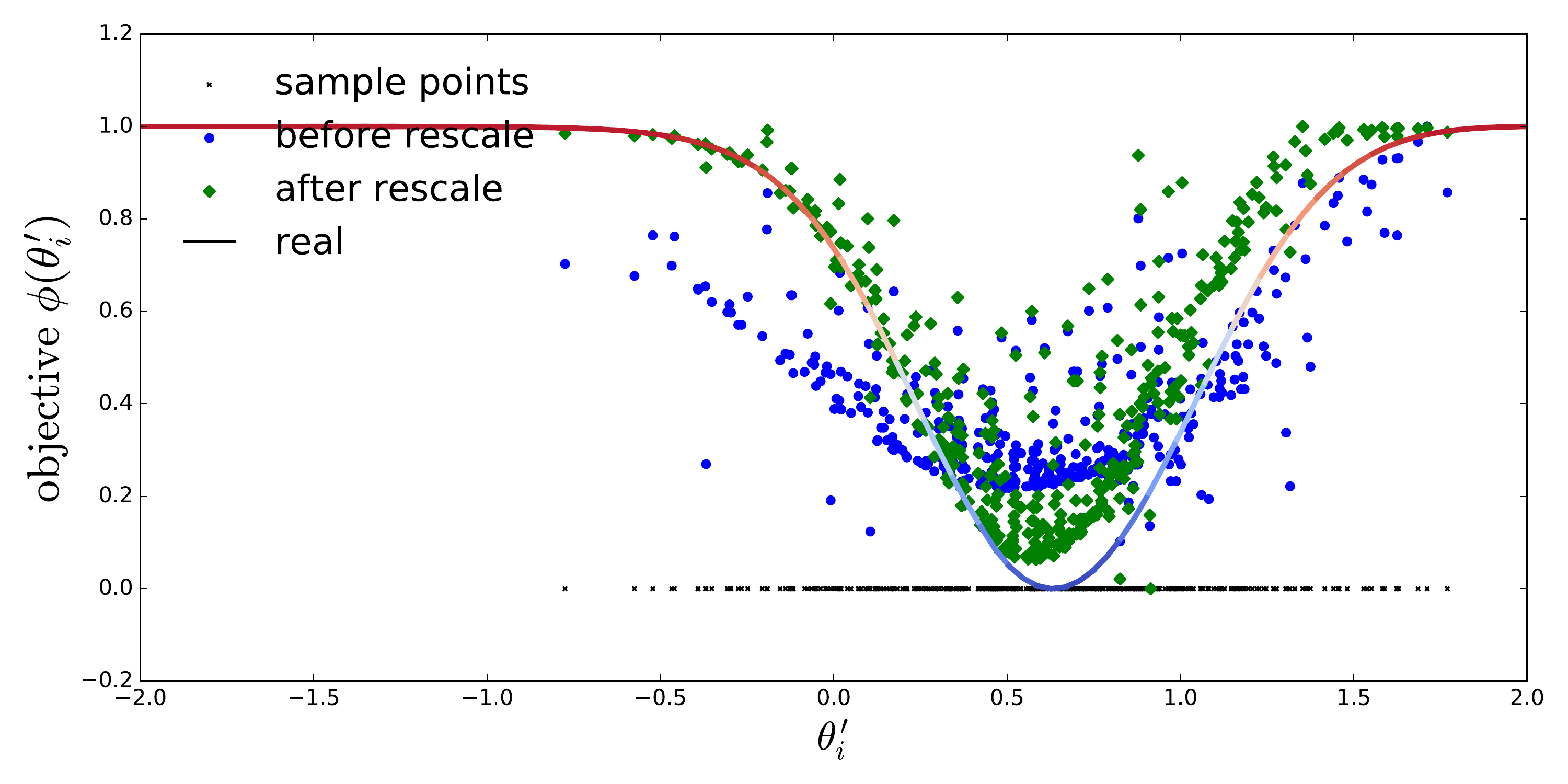}
		\caption{Estimate gradient through sampling from single distribution}\label{fig:estgrad2d}
	\end{figure}

	The original objective value $\phi(\theta_i')$ for each sample varies as the corresponding $p(\theta_c')$ changes, which introduces additional estimation bias. As shown by the blue dots in \prettyref{fig:estgrad2d}, the distribution of the objective values before rescale is significantly different from the distribution of the true objective values. From \prettyref{fig:estgrad3d}, it can be seen that as sample points deviate from the distribution center, their probability of being sampled also decreases, which means that the accuracy or confidence of the fitness of each sample $\phi(\varsigma_i')$ decreases with the decrease of $p(\theta_c')$. If the original objective is rescaled by its confidence $p(\theta_c')$, which is the probability of the appearance of the $\varsigma_i'$ given the existence of $\theta_{i}'$, the reconstructed objective values represented by the green square dot in \prettyref{fig:estgrad2d} is closer to the real $\phi(\theta_i')$, which obviously reduce the estimation bias.
	
	The above proof indicates that in the case of limited population size and a large number of neighboring populations, applying the rescaled gradient will keep the approximation bias to the level of a single population, resulting in a more accurate estimation of gradient information, empirical results also supported this conclusion. However, when the population size is large enough (e.g., thousands), this approach may not result in additional accuracy improvements.
	
	The modified expression is also desirable for parallel computing, as only the perturbation of the neighboring populations is needed, which can be easily obtained through communication among processes, and the probabilities can be calculated in a distributed approach.
	
	\subsection{Elitist adaptation Techniques}
	
	The performance of NES is sensitive to hyper-parameters, and the learning rate is usually the most critical hyper-parameter of NES. 
	Thus, an elitist adaptation method for the learning rate is applied in this paper. 
	First, a list of learning rates is linearly selected in the neighborhood of the original learning rate  $\eta_\alpha$ as:
	\begin{equation}\label{eqn:lrcad}
		\begin{split}
			\eta_{cad} &= \{clip((1+0.1k)\eta_\alpha, \eta_{\alpha min}, \eta_{\alpha max}):k \in \mathbb{Z}, \\
			&\quad -l/2\le k \le l/2\}, 
		\end{split}
	\end{equation}
	where, $\eta_{cad}\in \mathbb{R}^{m+1}$. The $\eta_{\alpha min}$ and $\eta_{\alpha max}$ are the minimum and maximum value of  $\eta_\alpha$. $l$ is the size of perturbations which is clipped by $clip(\cdot)$. 
	To evaluate the quality of the candidate learning rates, the evaluation function $G_i(\cdot)$ is defined:
	\begin{equation}\label{eqn:lrquality}
		G_i(\eta_{cad}) = 
		\begin{pmatrix}
			F_i(\theta_i + \eta_{cad}^{-l/2} g_{\theta_{i}}) - F_i(\theta_i + \eta_{\alpha}g_{\theta_i})\\
			\;\;\vdots  \\
			F_i(\theta_i + \eta_{cad}^{l/2} g_{\theta_{i}}) - F_i(\theta_i + \eta_{\alpha}g_{\theta_i})
		\end{pmatrix},
	\end{equation}
	where $\eta_{cad}^k$ is the $k$th sampled learning rate of the candidate list. The gradient $g_{\theta_i}$ is kept after evaluation. Therefore, by comparing the candidate learning rates with the original one, the next update can be better than the previous one.
	Considering \emph{peer pressure}, each missile is assigned the same learning rate.
	The learning rate of the next generation is obtained by
	\begin{equation}\label{eqn:lradaptation}
		\eta_\alpha' = \argmax_{\eta_{cad}^k}(\sum_{i=1} ^{n}G_i(\eta_{cad})).
	\end{equation}

	A similar approach is employed to obtain the optimal $\Xi_{d1}^*$ during the training process.
	\begin{equation} \label{eqn:optimaltheta}
		\Xi_{d1}^* = \argmax_{\Xi_{d1}^k}\begin{pmatrix}
			H(\Xi_{d1}^1)\\
			\;\;\vdots \\
			H(\Xi_{d1}^h)\\
		\end{pmatrix}
	\end{equation}
	where, $\Xi_{d1}^k$ is uniformly sampled from the region $[-\pi, \pi]$.
	$H(\cdot)$ is the fitness function of sampled LOS angle that is defined as
	\begin{equation}\label{eqn:lrquality}
		H(\Xi_{d1}^k) =  F_{tot}(\theta_{init})|_{\Xi_{d1}=\Xi_{d1}^k},
	\end{equation}
	where $\theta_{init}=[\theta_{i}^{init}]$ is the joint initial individual parameters. In this way, the desired impact angles are established automatically.
	
	A rank-based fitness shaping method that is in the same spirit as the one proposed in \cite{wierstra2014natural} is employed in shaping the raw fitness. Conventionally, we still let $F_i(\cdot)$ denote the fitness function after shaping. Another technique called mirrored sampling \cite{brockhoff2010mirrored} is also applied for sampling parameter perturbations.

	\section{Hybrid co-evolutionary cooperative guidance algorithm }
	
	To achieve coordinated attack, the natural co-evolutionary strategy is applied to optimize the parameter matrices $\theta_{i}=[W_{3i} , W_{2i}, W_{1i}]$ of the neural network controller.

	The univariate Gaussian distribution with zero means and standard deviation $\sigma$  is used to sample perturbations. According to \prettyref{eqn:gradient5}, it can be obtained that:
	\begin{dmath} \label{eqn:thisgradient}
		g_{\theta_i} = \mathbb{E}_{\epsilon_i\sim N(0, \sigma^2)}\left\{\triangledown_{\theta_i }{logp(\theta_i')}F_i(\varsigma_i')\prod_{c\in \mathcal{N}_i}p(\theta_c')\right\}\\
		= \frac{1}{m\sigma^2}\sum_{i=1}^{m}F_i(\varsigma_i')\epsilon_i \prod_{c\in \mathcal{N}_i}p(\epsilon_c).
	\end{dmath}

	The complete implementation algorithm of the proposed guidance law is shown in Algorithm \prettyref{alg:HCCGL}.
	The conceptual diagram in \prettyref{fig:HCCGL} figuratively revealed the parallel simulation process. A master-slave (or fully-distributed) model \cite{gong2015distributed}\cite{mendiburu2005parallel} is used for large-scale parallel computation. In this case, each population is evaluated in a separate process and the results of the ecosystem are aggregated to calculate the rescaled gradient \prettyref{eqn:thisgradient} and sent to produce guided generations. The sampled generations are then distributed to each parallel process, and the gradient is recalculated and updated.

	\begin{algorithm}
		\caption{Hybrid Cooperative Co-Evolutionary Guidance Law (HCCGL)}\label{alg:HCCGL}
		
		\begin{algorithmic}
			\Require $\eta_\alpha$, $\eta, \sigma, \theta_{init} = [\theta_{i}^{init}], $agent number n.
			
			Sample $[\Xi_{d1}^k]\in \mathbb{R}^{h} \sim U(-\pi, \pi)$, obtain $\Xi_{d1}^*$ using \prettyref{eqn:optimaltheta}
			\Repeat	
			\For{ k= 1... m}
			\State  Sample group of individuals:
			\State  $\epsilon^k = \{\epsilon^k_i \sim N(0, \sigma^2I): i\in\{1,...,n\}\},$
			\State $\varsigma^k=\{\varsigma_i^k=  \{\theta_i^k, \theta_j^k : j\in \mathcal{N}_i\}: i\in\{1,...,n\}\}$ 
			\State  evaluate fitness $F_i(\varsigma_i^k)$, for $i\in\{1,...,n\}$
			\EndFor
			
			\For{each agent i = 1... n}
			\State calculate natural gradient:
			\State $g_{\theta_i}\leftarrow \frac{1}{m\sigma^2}\sum_{k=1}^{m}F_i(\varsigma_i^k)\epsilon_i^k \prod_{c\in \mathcal{N}_i}p(\epsilon_c^k)$
			\State  $\theta_i\leftarrow \theta_i + \eta_{\alpha}\cdot g_{\theta_i}$
			\EndFor
			\If{time for adaptation}
			\State sample $\eta_{cadi}$ using \prettyref{eqn:lrcad}
			\State  $\eta_\alpha \leftarrow \argmax_{\eta_{cad}^k}(\sum_{i=1} ^{n}G_i(\eta_{cad}))$
			\EndIf
			\Until{stopping criterion is met}
		\end{algorithmic}
	\end{algorithm}

	\begin{figure}[htbp]
		\centering
		\adjincludegraphics[height=\linewidth, trim={0, 0.2\width, 0, 0.0\width}, clip]{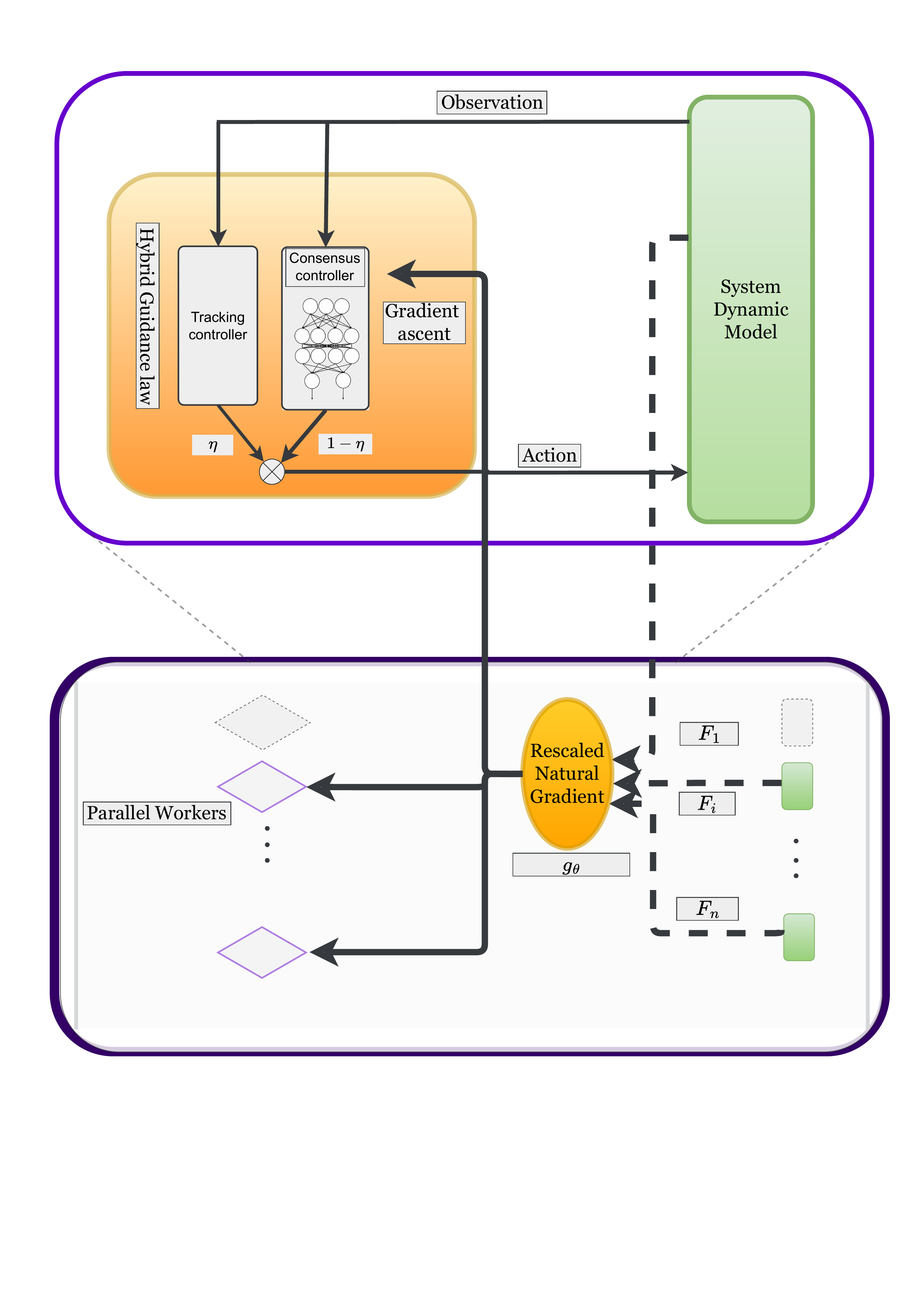}
		\caption{Conceptual framework of our proposed HCCGL, the upper box connected by dotted lines is a detailed expansion of the evaluation and evolutionary processes in the lower box.}\label{fig:HCCGL}
	\end{figure}
	
	\begin{theorem}
		Under the control policy \prettyref{eqn:composedcommand} and the update strategy shown in Algorithm \prettyref{alg:HCCGL}, by selecting an appropriate learning rate, sampling variance, and population size, the obtained control policy will converge to a small neighborhood of the optimal control policy when $l\to\infin$.
	\end{theorem}
	
	\begin{proof}
		The approximation error of the control policy at the \emph{l}th iteration is defined by
		\begin{equation}\label{eqn:policyapproerr}
			E_{ui}^{[l]}=u_i^{[l]} - u_i^{*},
		\end{equation}
		where $u_i^{*}$ is the optimal control policy of the \emph{i}th agent, and we have
		\begin{equation}\label{eqn:combinedpolicy}
			u_i^{[l]} = (1-\eta)u_{pi} + \eta u_{ei}^{[l]}.
		\end{equation}
		The control policy of the neural network is represented by its parameter set. Since a neural network with a single hidden layer can approximate a multivariate continuous function with arbitrary precision \cite{trennMultilayerPerceptronsApproximation2008}, which implies that a single hidden layer perceptron with sufficient units is equivalent to the neural network with three hidden layers used in this work. In this way, the neural network controller can be represented by
		\begin{equation}\label{eqn:neuralcontrollerpolicy}
			u_{ei}^{[l]} = W_i^{[l]T},
		\end{equation}
		which is a column matrix and the activation function parameters are regarded as constants.
		By combining \prettyref{eqn:neuralcontrollerpolicy} and \prettyref{eqn:combinedpolicy} and substituting it into \prettyref{eqn:policyapproerr}, we have
		\begin{equation}\label{eqn:errorAtl}
			E_{ui}^{[l]}=(1-\eta)u_{pi} + \eta W_i^{[l]T} - u_i^{*},
		\end{equation}
		\begin{equation}\label{eqn:errorAtl-1}
			E_{ui}^{[l-1]}=(1-\eta)u_{pi} + \eta W_i^{[l-1]T} - u_i^{*}.
		\end{equation}
		Further combining \prettyref{eqn:errorAtl} and \prettyref{eqn:errorAtl-1}, the term of fixed controllers are eliminated, and we have
		\begin{equation}
			\begin{aligned}
				E_{ui}^{[l]} - E_{ui}^{[l-1]} 
				&= 
				\eta(W_i^{[l]T}-W_i^{[l-1]T}) \\
				&=
				\eta\eta_\alpha g_{\theta_{i}}^{[l-1]}, \\
			\end{aligned}
		\end{equation}
		where $\eta_\alpha$ is the learning rate and $\eta$ the guidance gain, with $\eta_\alpha, \eta > 0$. The focus of this equation, $g_{\theta_{i}}^{[l-1]}$, is the policy update gradient at the \emph{l-1}th iteration, which follows by \prettyref{eqn:thisgradient}.
		Expanding this equation, we have
		\begin{equation}\label{eqn:errorDiff1}
			\begin{aligned}
				E_{ui}^{[l]} - E_{ui}^{[l-1]} 
				&= 
				\eta\eta_\alpha g_{\theta_{i}}^{[l-1]} \\
				&=
				\eta\eta_\alpha \frac{1}{m\sigma^2}\sum_{k=1}^{m}F_i(E_{ui}^{[k]})\epsilon_i^k \prod_{c\in \mathcal{N}_i}p(\epsilon_c^k) \\
				&=
				\eta\eta_\alpha \frac{1}{m\sigma^2}\sum_{k=1}^{m}F_i(E_{ui}^{[k]})*(E_{ui}^{[k]}-E_{ui}^{[l-1]}) \\[\jot]
				&\quad \prod_{c\in \mathcal{N}_i}p(\epsilon_c^k).
			\end{aligned}
		\end{equation}
		Note that in \prettyref{eqn:errorDiff1}, $F_i(E_{ui}^{[k]}):\mathbb{R}^p\to\mathbb{R}^p$ is the transformed fitness function for the evaluation of the policy error, with p as the number of parameters. Thus, it is different from the fitness function discussed in the previous sections, which evaluates the policy directly. It is assumed that $F_i(E_{ui}^{[k]})$ is fully differentiable to the policy controller, and
		\begin{equation}
			\frac{\partial F_i(|E_{ui}^{[k]}|)}{\partial |E_{ui}^{[k]}|} < 0,
		\end{equation} 
		considering that $|E_{ui}^{[k]}|$ represents the quality of the policy globally.
		
		In an effort to linearize the fitness evaluation function, Taylor's formula is utilized to expand the equation at $|E_{ui}^{[l-1]}|$ and the higher order terms are ignored, and we obtain 
		\begin{equation}\label{eqn:fitnessAppro}
			F_i(|E_{ui}^{[k]}|) = G_i^{[l-1]}[|E_{ui}^{[k]}| - |E_{ui}^{[l-1]}|] + F_i(|E_{ui}^{[l-1]}|),
		\end{equation}
		where $G_i^{[l-1]} \in \mathbb{R}^{p\cross p}$ is a diagonal Jacobian matrix defined by
		\begin{equation}
			G_i^{[l-1]} = \eval{\pdv{F_i(|E_{ui}^{[k]}|)}{|E_{ui}^{[k]}|}}_{|E_{ui}^{[l-1]}|},
		\end{equation}
		with negative entries and p as the number of total parameters.
		Since $|E_{ui}^{[k]}|$ is located within a tiny vicinity of $|E_{ui}^{[l-1]}|$, \prettyref{eqn:fitnessAppro} is of considerable accuracy.
		
		Then, by taking the absolute value of the approximation error and substituting \prettyref{eqn:fitnessAppro} into \prettyref{eqn:errorDiff1}, and considering $\Delta E_i^k=|E_{ui}^{[k]}| - |E_{ui}^{[l-1]}|$ we obtain
		\begin{equation}
			\begin{aligned}
				\Delta|E_{ui}^{[l]}|
				&=
				|E_{ui}^{[l]}| - |E_{ui}^{[l-1]}| \\
				&= 
				\eta\eta_\alpha \frac{1}{m\sigma^2}\sum_{k=1}^{m}[G_i^{[l-1]}\Delta E_i^k +  F_i(|E_{ui}^{[l-1]}|)]* \\
				&\quad \Delta E_i^k* \prod_{c\in \mathcal{N}_i}p(\epsilon_c^k) \\
				&= 
				\eta\eta_\alpha \frac{1}{m\sigma^2}G_i^{[l-1]} \sum_{k=1}^{m}[\Delta E_i^k* \Delta E_i^k + \\
				&\quad F_i(|E_{ui}^{[l-1]}|)*[|E_{ui}^{[k]}| - |E_{ui}^{[l-1]}|]] *\prod_{c\in \mathcal{N}_i}p(\epsilon_c^k) \\
				&=
				\eta\eta_\alpha \frac{1}{m\sigma^2}[G_i^{[l-1]} \sum_{k=1}^{m}\Delta E_i^k* \Delta E_i^k * \prod_{c\in \mathcal{N}_i}p(\epsilon_c^k) \\
				&\quad + F_i(|E_{ui}^{[l-1]}|)*\sum_{k=1}^{m}\Delta E_i^k * \prod_{c\in \mathcal{N}_i}p(\epsilon_c^k)].  \\
			\end{aligned}
		\end{equation}
		For brevity, we define $A_i^{k}, P_i^k$, and $B_i^k$ by
		\begin{equation}
			\begin{aligned}
				A_i^{k} &= \Delta E_i^k*\Delta E_i^k; \\
				P_i^k &= \prod_{c\in \mathcal{N}_i}p(\epsilon_c^k);\\
				B_i^k &=  \sum_{k=1}^{m}|E_{ui}^{[k]}|*P_i^k - \sum_{k=1}^{m}|E_{ui}^{[l-1]}|*P_i^k,
			\end{aligned}
		\end{equation}
		such that 
		\begin{equation}
			\begin{aligned}
				\Delta|E_{ui}^{[l]}|
				&= 	
				\eta\eta_\alpha \frac{1}{m\sigma^2}[G_i^{[l-1]}\sum_{k=1}^{m}A_i^k * P_i^k + \\
				&\quad F_i(|E_{ui}^{[l-1]}|)*B_i^k].
			\end{aligned}
		\end{equation}
		
		Since $|E_{ui}^{[k]}|$ is sampled from an unbiased normal distribution which is centered at $|E_{ui}^{[l-1]}|$, as shown in the analysis of \prettyref{sec:co-evoAnalysis}, we have
		\begin{equation}
			B_i^k \to  0, \quad\text{as}\,\, m\to\infin .
		\end{equation}
		Also, from the matrix Hadamard product we have
		\begin{equation}
			A_i^k > 0,
		\end{equation}
		and 
		\begin{equation}
			P_i^k > 0.
		\end{equation}
		$G_i^{[l-1]}$ is negative definite. Given sufficient large $m$, it is evident that
		\begin{equation}
			\Delta|E_{ui}^{[l]}| < 0, \quad l=1,2,...\,.
		\end{equation}
		Therefore, by adjusting the learning rate $\eta_\alpha$ attentively, the approximation error can be decreased to a considerably small range $\delta_e$, such that
		\begin{equation}
			\begin{aligned}
				\lim_{l\to\infin}|E_{ui}^{[l]}| = \delta_e,\\
				\delta_e \to 0, \quad \text{as}\ m \to \infin.
			\end{aligned}
		\end{equation}
		Thus, the control policy $u_i$ converges to a small neighborhood of the optimal control policy $u_i^*$, resulting in a stabilizing control system.

	\end{proof}

	\section{Simulations and analysis}
	\label{sec:simulations}
	
	To verify the validity of the proposed method, a variety of simulations based on the cooperative guidance framework are designed. Both cases with stationary target and maneuvering target are simulated. Further, comparison experiments are performed to fully demonstrate the superiority of the proposed guidance method.

	\subsection{Paremeter setup}
	The acceleration constraint and velocity constraint of the missiles are listed in \prettyref{tab:ExpSetup}.
	The hyper-parameters of the algorithm are listed in \prettyref{tab:HyperParameters}.

	Now that frameskip has been extensively employed in continuous control problems\cite{salimans2017evolution}. In this work, this parameter of frameskip is set to 12 for case 1 and case 2, and 40 for case 3. 
	Appropriate adjustment of this parameter will facilitate the training process without affecting the final results.

	\begin{table}[]
		\centering
		\caption{Constraints  of the missiles}
		\label{tab:ExpSetup}
		\begin{tabular}{@{}ll@{}}
			\hline\hline
			Parameter                                          & Value \\ \midrule
			maximum lateral overload (g) , $a_{lmax}$ & 50    \\
			maximum trust overload (g), $a_{vmax}$    & 5     \\
			Upper bound of velocity (m/s),     $V_{max}$           & 900   \\
			Lower bound of velocity (m/s), $V_{min}$   & 350   \\ \hline\hline
		\end{tabular}
	\end{table}
	
	\begin{table}[]
		\centering
		\caption{Hyper-parameters of the cooperative guidance algorithm}
		\label{tab:HyperParameters}
		\begin{tabular}{@{}ll@{}}
			\hline\hline
			Parameter                                            & Value \\ \midrule
			simulation step (ms), $\tau$                         & 5     \\
			guidance gain, $\eta$                                & 0.3   \\
			Initial learning rate, $\eta_\alpha$                 & 0.015 \\
			standard deviation for sampling population, $\sigma$ & 0.2   \\
			size of learning rate adaptation, $l$                & 20    \\
			size of population, m                                & 140   \\
			adaptation cycle, $\rho$                             & 50    \\
			navigation constant, $\beta$                             & 4     \\
			$k_a$                                                & 1     \\
			$k_t$                                                & 0.2   \\
			$\xi_a$                                              & 10    \\
			$\xi_t$                                              & 1     \\
			$\lambda_a$                                          & 4000    \\
			$\lambda_t$                                          & 2000     \\
			$\beta_a$                                            & 10    \\
			$\beta_t$                                            & 2     \\ \hline\hline
		\end{tabular}
	\end{table}

	\subsection{Case 1: Comparison Experiments}
	
	In this section, the proposed guidance law is compared with the time and space cooperative guidance law (TASCGL) proposed in \cite{lyu2019multiple}, which considers the space and time cooperative guidance under the distributed communication topology. However, different from the method proposed in this work, the compared method is susceptible and brittle to the initial conditions. Therefore, in order to verify the generalization ability of the control methods, a uniform initial condition was adopted in the comparison simulation, which differs slightly from the initial condition in the comparison method. The initial conditions as shown in \prettyref{tab:InitialConditionofCase1}.  
	Four missiles are engaged in the cooperative scenario with different desired relative impact angles $\delta_{d}^i$ as $20\degree, 60\degree$, and $30 \degree$, for each $i=1,2,3$, respectively. The target is located at (9500, 9000)m.
	
	Although the reference method is primarily designed for directed topology, it can be well extended to an undirected topology condition, thus in order to conduct effective comparison experiments, we additionally implemented a comparison experiment under an undirected topology the same as the one used in the proposed method. We use TASCGL$^a$ and TASCGL$^b$ to denote the comparative experiments performed under directed and undirected communication topologies, respectively.
	\begin{table}[th]
		\centering
		\caption{Initial conditions of case 1.}
		\label{tab:InitialConditionofCase1}
		\begin{tabular}{@{}llll@{}}
			\hline\hline
			Missile                                   & Position (m) & Flight-path  & Velocity\\ 
			&                       &           Angle ($\degree$)   &     (m/s)      \\  \midrule
			$M_1$                                     & (1900, 17000)        & -25                               & 700                   \\
			$M_2$                                     & (1500, 13000)        & 0                                 & 650                   \\
			$M_3$                                     & (1400, 4000)         & 5                                 & 700                   \\
			$M_4$                                     & (3000, 1300)         & 10                                & 680                   \\  \hline\hline
		\end{tabular}
	\end{table}

	
	\begin{figure}[ht!]
		\centering
		\begin{subfigure}[ht]{0.45\textwidth}
			\centering
			\adjustbox{trim={.0\width} {.1\height} {.08\width} {.1\height},clip}%
			{\includegraphics[width=\textwidth]{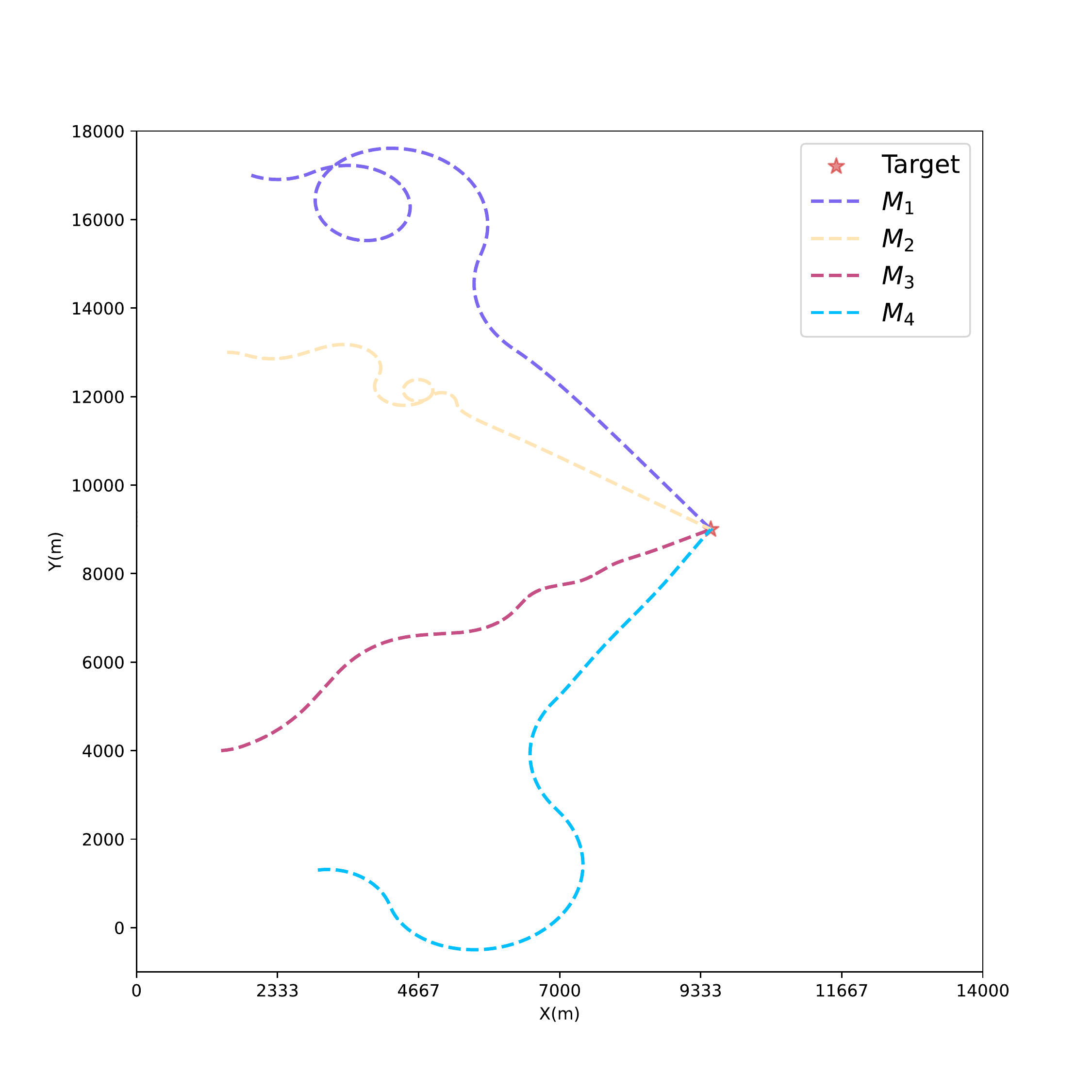}}
			
			\caption{TASCGL$^a$}\label{fig:tascglTrajCase1_a}
		\end{subfigure}
		\vfill
		\begin{subfigure}[ht]{0.45\textwidth}
			\centering
			\adjustbox{trim={.0\width} {.1\height} {.08\width} {.1\height},clip}{	\includegraphics[width=\linewidth]{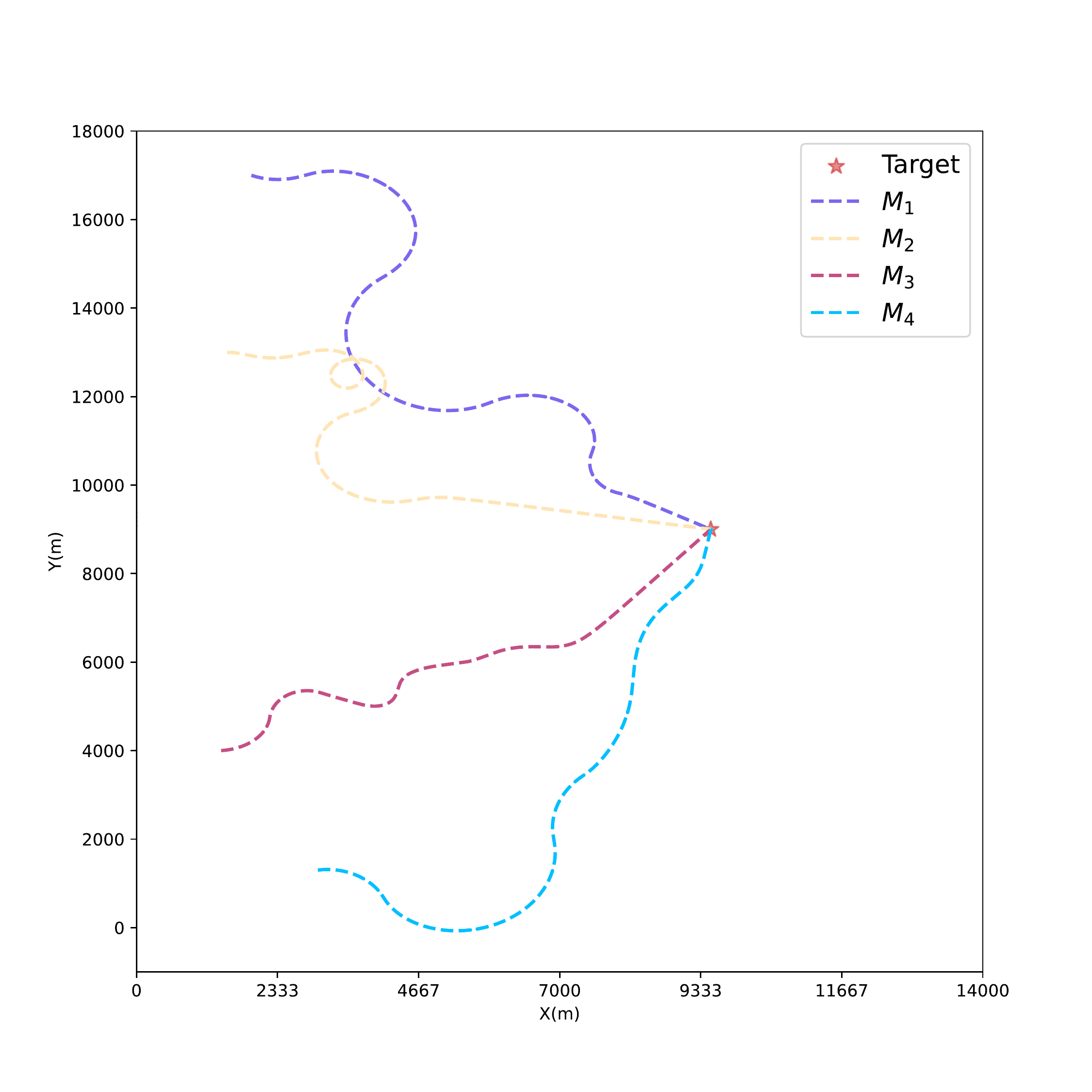}}
			\caption{TASCGL$^b$}\label{fig:tascglTrajCase1_b}
		\end{subfigure}
		\vfill
		\begin{subfigure}[ht]{0.45\textwidth}
			\centering
			\adjustbox{trim={.0\width} {.1\height} {.08\width} {.1\height},clip}{
				\includegraphics[width=\textwidth]{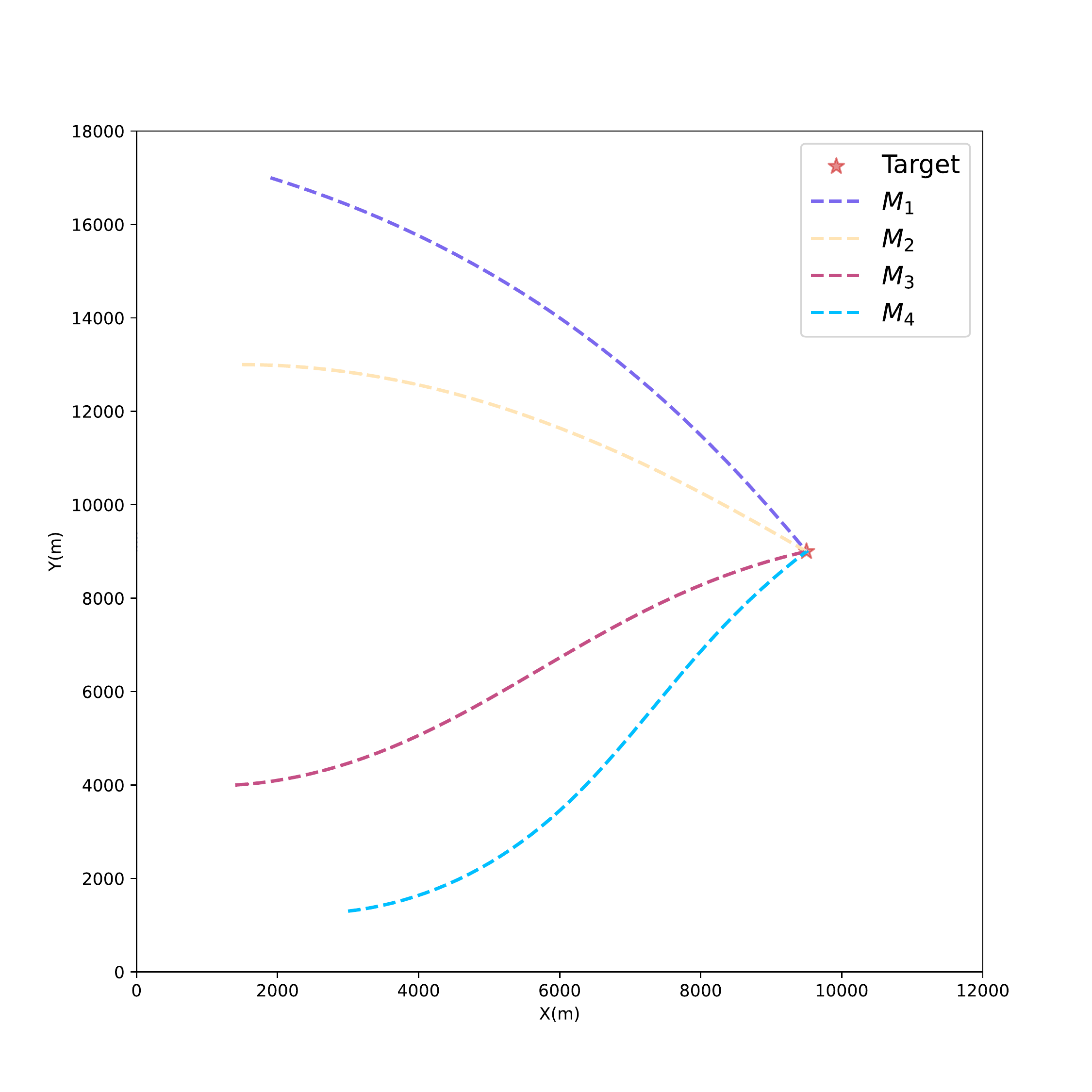}}
			\caption{HCCGL}\label{fig:HCCGLTrajCase1}
		\end{subfigure}
		\caption{Trajectories of the two methods}\label{fig:trajectoryCase1}
	\end{figure}
	
	\begin{figure}
		\centering
		\begin{subfigure}[ht]{0.5\textwidth}
			\centering
			\includegraphics[width=\textwidth]{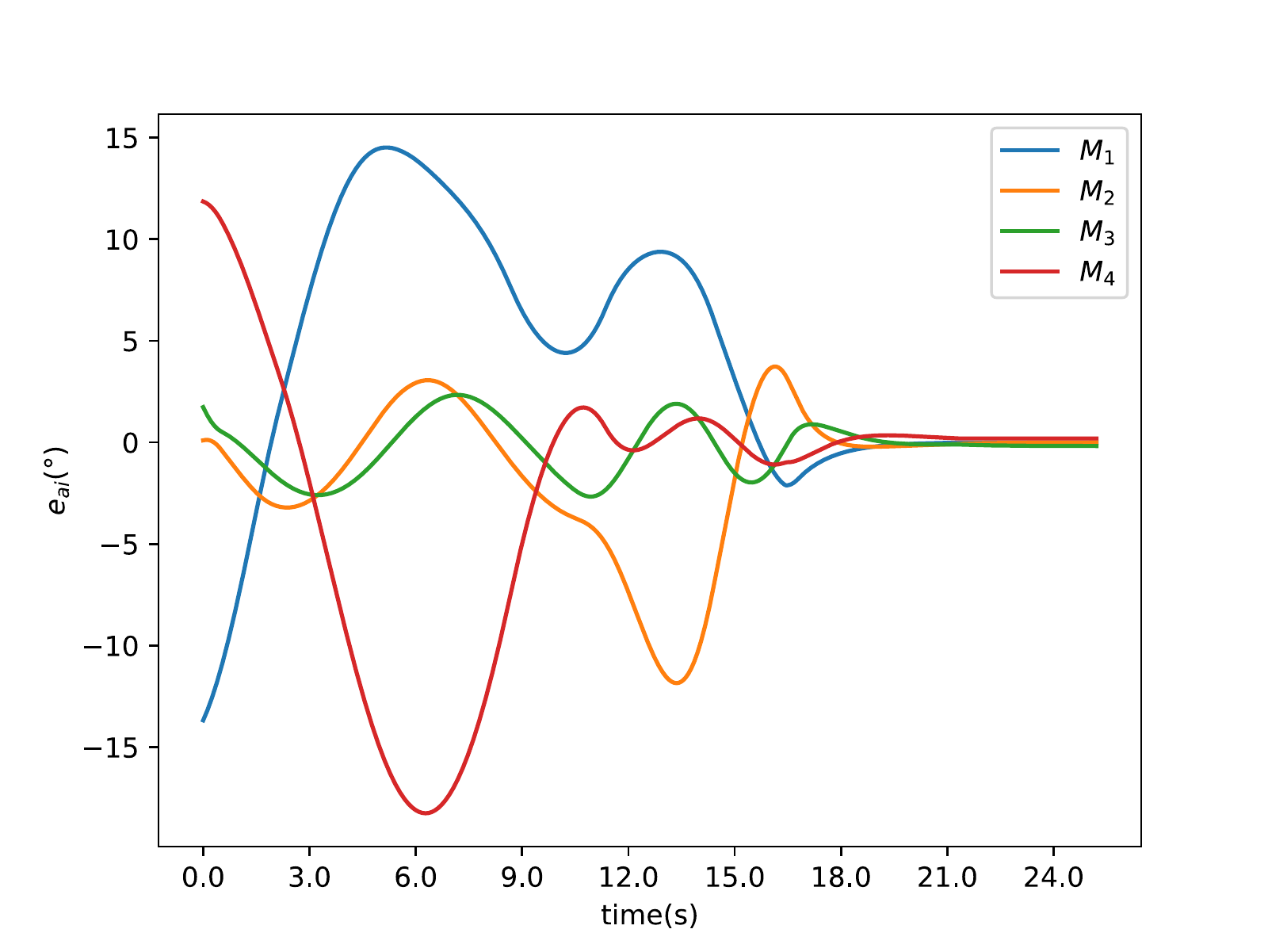}
			\caption{TASCGL$^a$}\label{fig:tascglLoserrorCase1_a}
		\end{subfigure}
		\begin{subfigure}[ht]{0.5\textwidth}
			\centering
			\includegraphics[width=\textwidth]{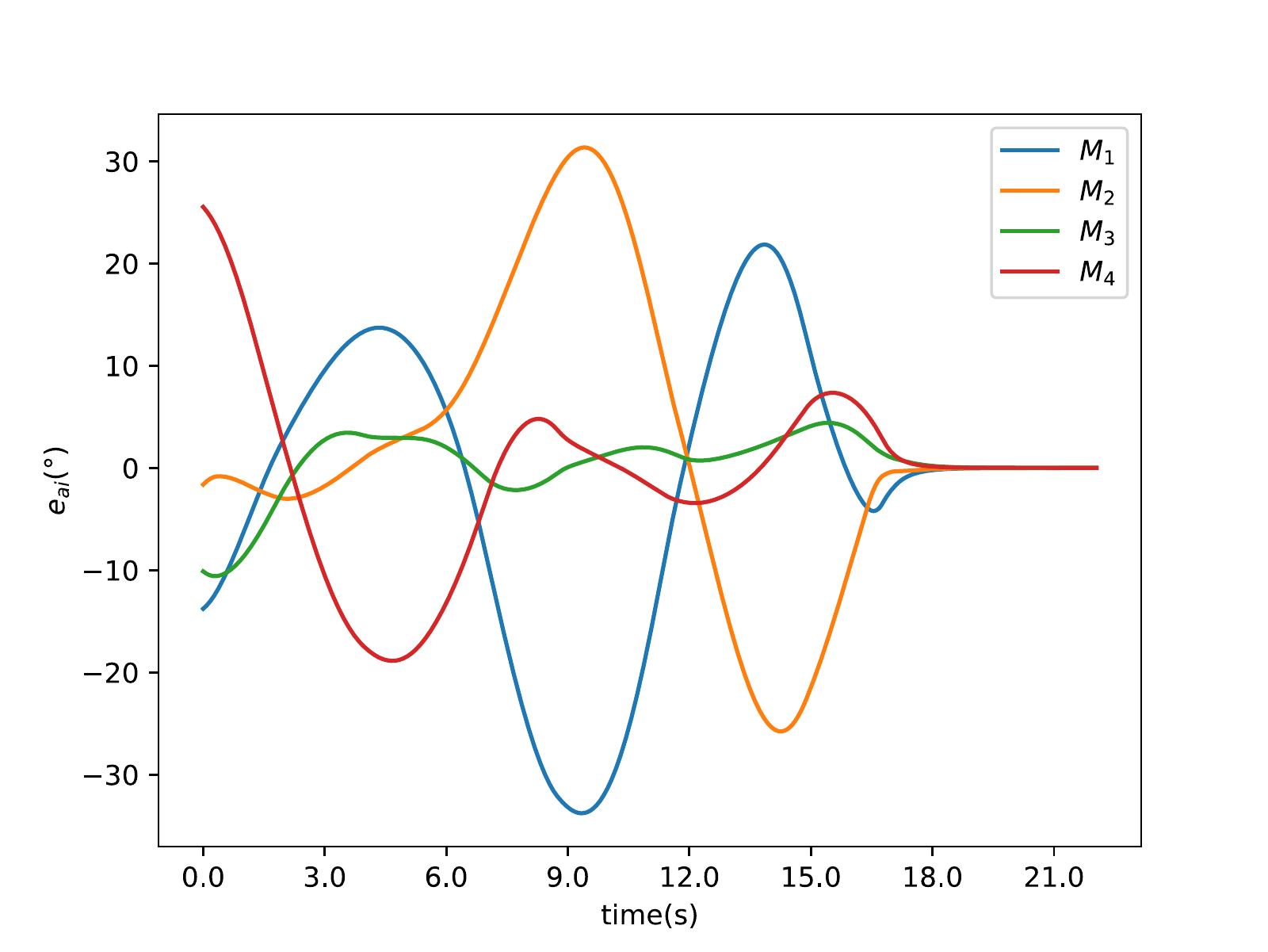}
			\caption{TASCGL$^b$}\label{fig:tascglLoserrorCase1_b}
		\end{subfigure}
		\begin{subfigure}[ht]{0.5\textwidth}
			\centering
			\includegraphics[width=\textwidth]{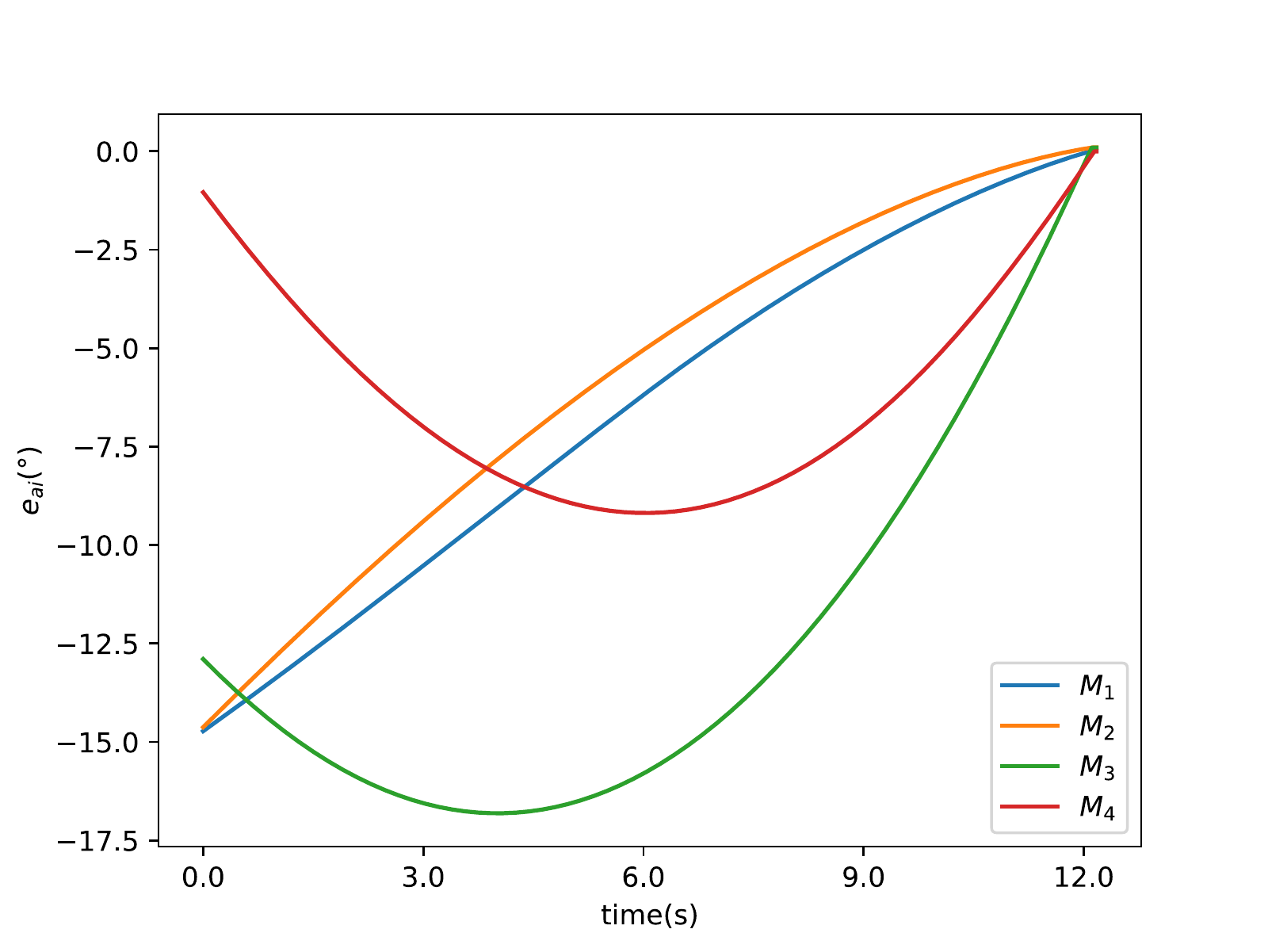}
			\caption{HCCGL}\label{fig:HCCGLLoserrorCase1}
		\end{subfigure}
		\caption{Consensus angle error profiles of the two methods}\label{fig:consensusAngleError}
	\end{figure}
	
	\begin{figure}
		\centering
		\begin{subfigure}[ht]{0.5\textwidth}
			\centering
			\includegraphics[width=\textwidth]{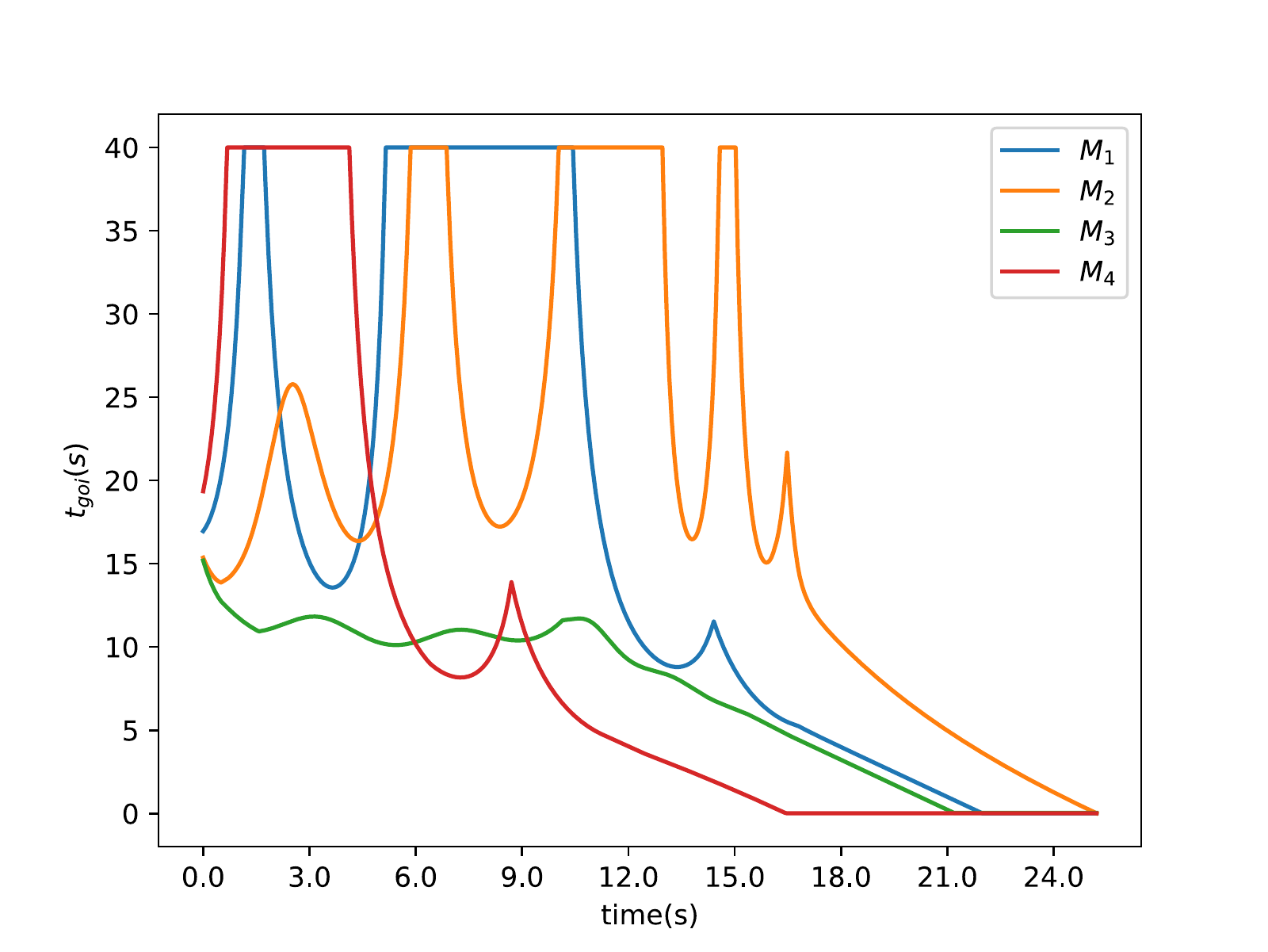}
			\caption{TASCGL$^a$}\label{fig:tascglTTGerrorCase1_a}
		\end{subfigure}
		\begin{subfigure}[ht]{0.5\textwidth}
			\centering
			\includegraphics[width=\textwidth]{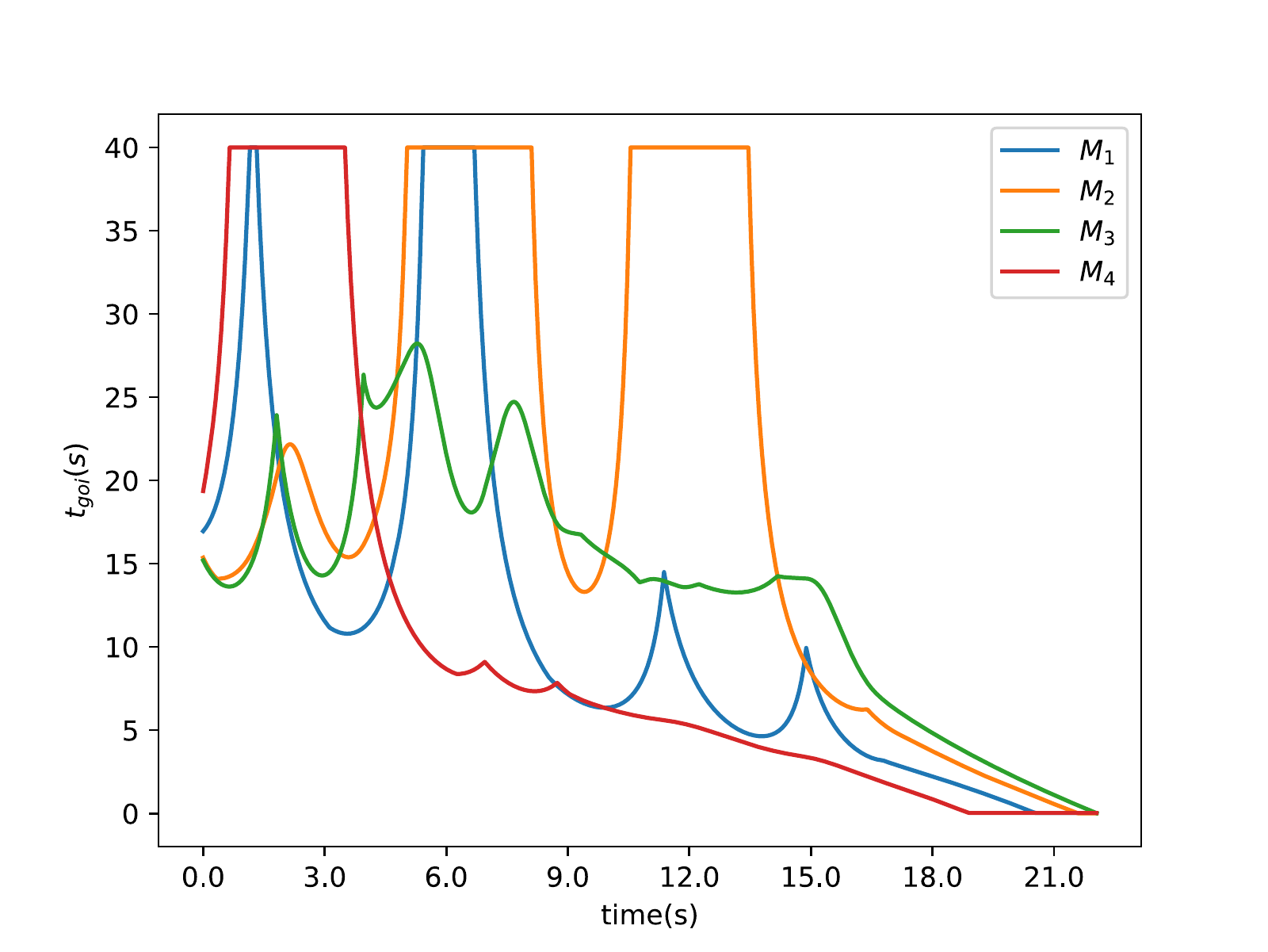}
			\caption{TASCGL$^b$}\label{fig:tascglTTGerrorCase1_b}
		\end{subfigure}
		\begin{subfigure}[ht]{0.5\textwidth}
			\centering
			\includegraphics[width=\textwidth]{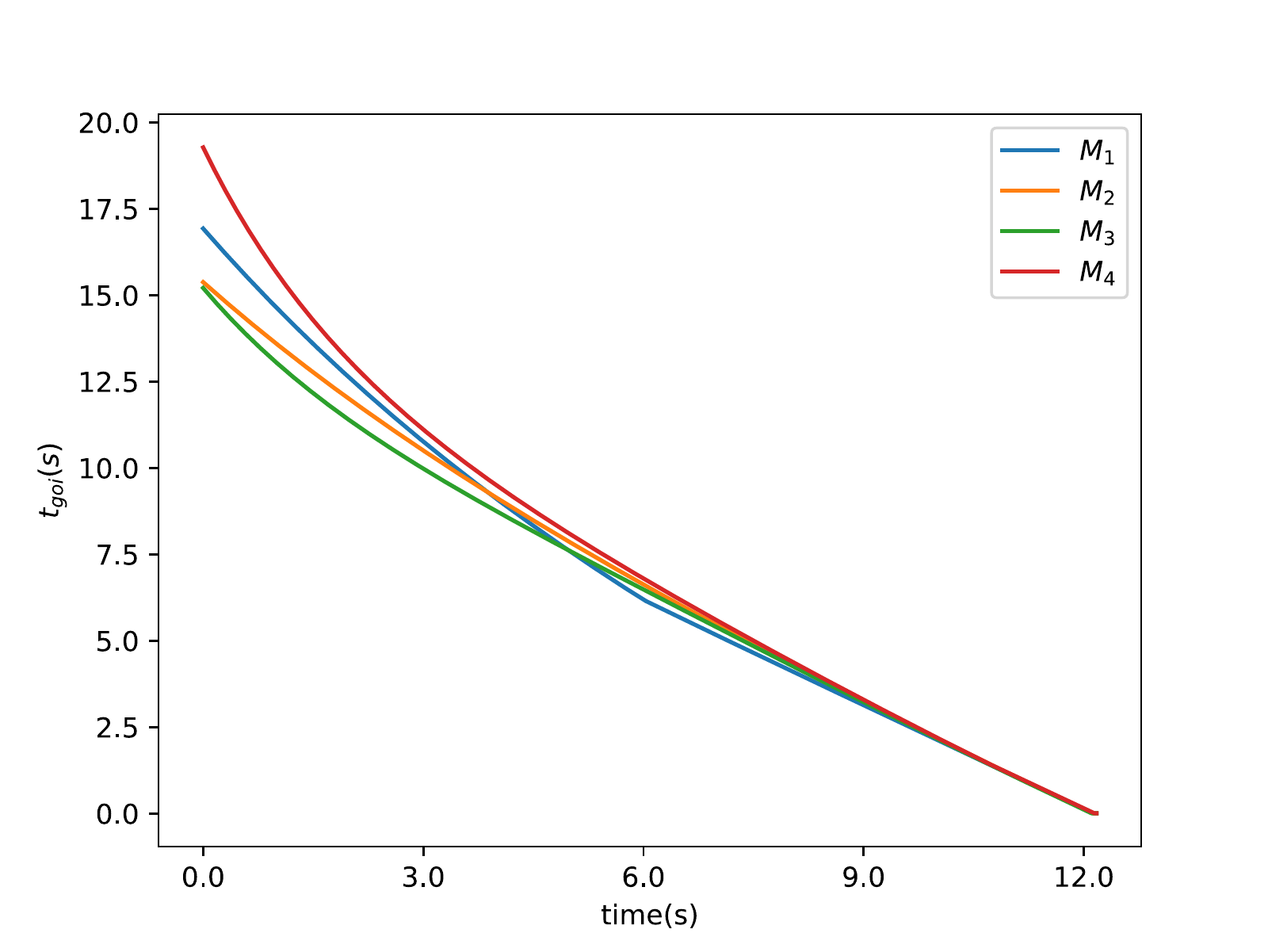}
			\caption{HCCGL}\label{fig:HCCGLTTGerrorCase1}
		\end{subfigure}
		\caption{Time-to-go profiles of the two methods}\label{fig:timeToGo}
	\end{figure}
	
	\begin{figure}[htb]
		\centering
		\begin{subfigure}[ht]{0.5\textwidth}
			\centering
			\includegraphics[width=\textwidth]{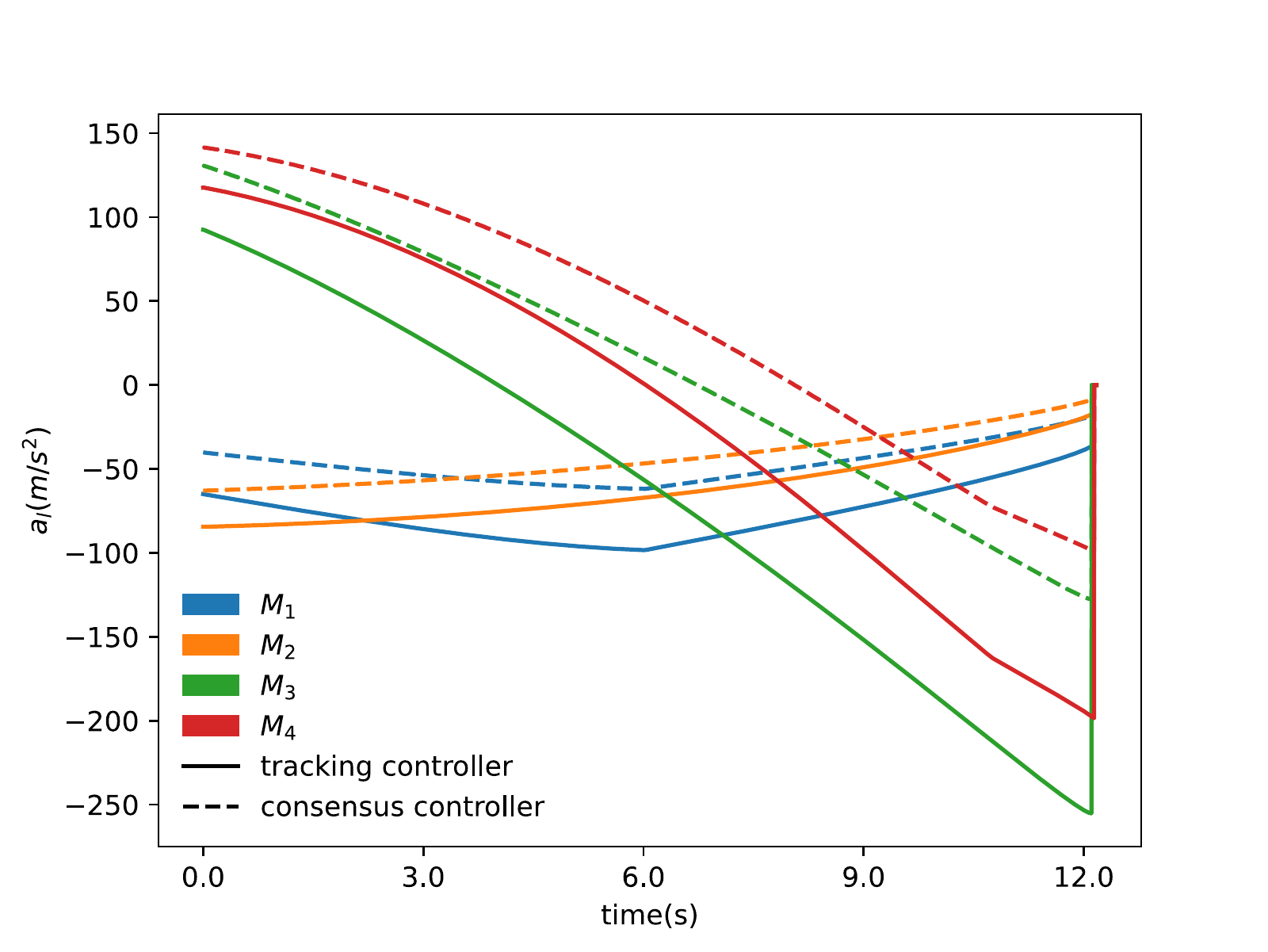}
		\end{subfigure}
		\begin{subfigure}[ht]{0.5\textwidth}
			\centering
			\includegraphics[width=\textwidth]{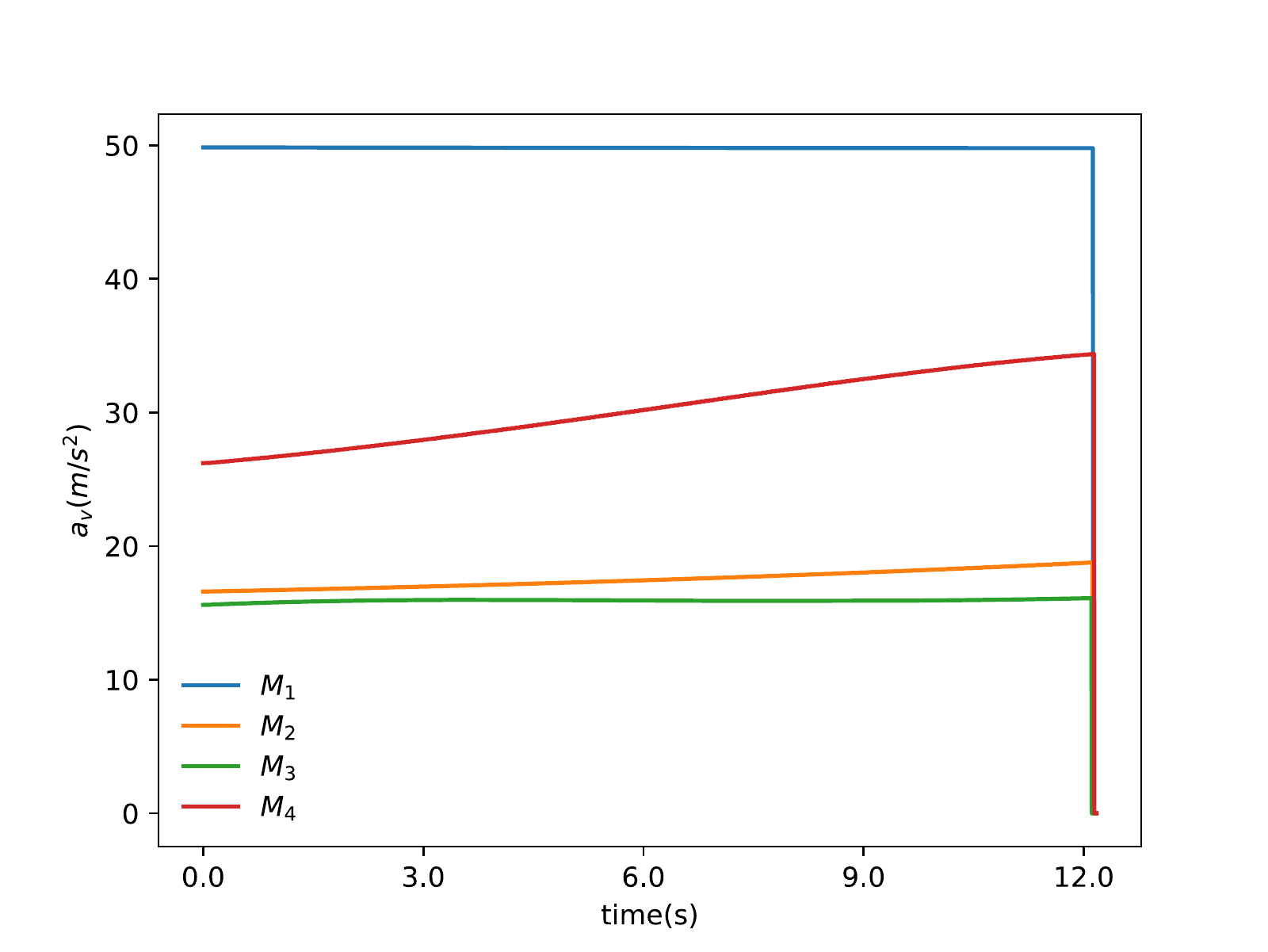}
		\end{subfigure}
		\caption{Decomposition of acceleration commands in Case 1}\label{fig:decoCase1}
	\end{figure}

	\prettyref{fig:trajectoryCase1} shows the trajectories of the two guidance laws. As depicted in the figure, the trajectory of TASCGL is  twisted at the initial stage, as the missiles try to consensus their LOS angles and velocities. In comparison, the trajectory of the proposed HCCGL exhibited better damping performance with no oscillations. 
	
	It can be seen from \prettyref{tab:ResultForCase1}
	that the Zero-Effort Miss (ZEM) and the consensus angle error for both guidance laws have achieved competitive final accuracy. 
	The consensus time error of TASCGL was up to 5 seconds under both directed and undirected topologies, whereas the proposed method achieved an error of less than 0.1 seconds.
	Further analysis of the velocity curve shows that in the case of TASCGL, the velocities are prohibited from reaching their ideal values due to the velocity boundary, which is not considered in its design, thus leading to desynchronization in impact time.
	The profiles of the two methods are shown in \prettyref{fig:consensusAngleError} and \prettyref{fig:timeToGo}, it can be observed that the flight time of all missiles under HCCGL trends to be identical.
	For HCCGL, the decomposition of acceleration commands is shown in \prettyref{fig:decoCase1}. The left figure shows the decomposition of lateral accelerations, in which the solid line represents the command from the tracking controller while the dashed line represents the command from the consensus controller before weighing. Since the tracking part is derived from proportional navigation, the vertical acceleration shown on the right one is completely derived from the consensus controller. The two parts of accelerations have similar trends but do not coincide, demonstrating the effectiveness of the consensus controller, which is trained with the improved co-evolutionary strategy.
	
	The result reveals that the proposed guidance law outperforms the compared method with higher precision in consensus performance and smoother trajectories. Moreover, as the traditional guidance law is usually constrained to boundary conditions and missile’s superb maneuverability, the proposed guidance law is more resilient to limited conditions and more intelligent to be aware of the time-varying states of missiles of collaboration.
	
	\begin{table}[htb]
		\centering
		\caption{Comparison results of two guidance laws in case 1.}
		\label{tab:ResultForCase1}
		\begin{tabular}{@{}p{0.15\columnwidth}p{0.15\columnwidth}llll@{}}
			\hline\hline
			Algorithm               & Index                         & $M_1$    & $M_2$    & $M_3$    & $M_4$   \\ \midrule
			\multirow{3}{*}{TASCGL$^a$} 
			& $e_t^i$(s) & 5.54 & 3.23    & -4.02    & -4.75   \\
			& $e_a^i$($\degree$) & -9.83\rm{E}-3 & -1.58\rm{E}-3 & -1.80\rm{E}-1 & 1.91\rm{E}-1 \\
			& ZEM(m) & 2.24\rm{E}-7  & 5.06\rm{E}-7  & -7.00\rm{E}-2    & 3.00\rm{E}-4 \\ \midrule
			\multirow{3}{*}{TASCGL$^b$} 
			& $e_t^i$(s) & 6.15\rm{E}-1 & 5.5\rm{E}-1   & 3.655    & -4.82   \\
			& $e_a^i$($\degree$) & -6.11\rm{E}-3 & -1.28\rm{E}-3 & -3.38\rm{E}-4 & 7.73\rm{E}-3 \\
			& ZEM(m) & 1.19\rm{E}-5  & 5.11\rm{E}-8  & 4.81\rm{E}-7    & 5.53\rm{E}-5 \\ \midrule
			\multirow{3}{*}{HCCGL}
			& $e_t^i$(s)  & -1.00\rm{E}-2   & 1.00\rm{E}-2    & -5.00\rm{E}-2  & 5.00\rm{E}-2    \\
			& $e_a^i$($\degree$) & 1.79\rm{E}-2  & 9.67\rm{E}-2  & 9.20\rm{E}-2 & 4.69\rm{E}-3 \\
			& ZEM(m) & 4.19\rm{E}-5  & 7.66\rm{E}-6  & 3.11\rm{E}-4  & 4.09\rm{E}-5 \\
			\hline\hline
		\end{tabular}
	\end{table}
	
	\subsection{Case 2: Non-stationary target}
	In this part, an engagement scenario with a non-stationary target is designed and simulated to verify the effectiveness of the proposed method against unknown dynamic target. The target is maneuvering with lateral acceleration $a_t = 5g\sin(\frac{\pi}{7}t)$ with its velocity fixed at $V_t= 130 m/s$,  and its initial flight-path angle $\alpha_T=162\degree$. Other initial conditions are the same in case 1. Simulation trajectory and the result can be seen in \prettyref{fig:trajCase2} and \prettyref{tab:ResultForCase2}.
	
	
	\begin{figure}[htbp]
		\centering
		\includegraphics[width=0.45\textwidth]{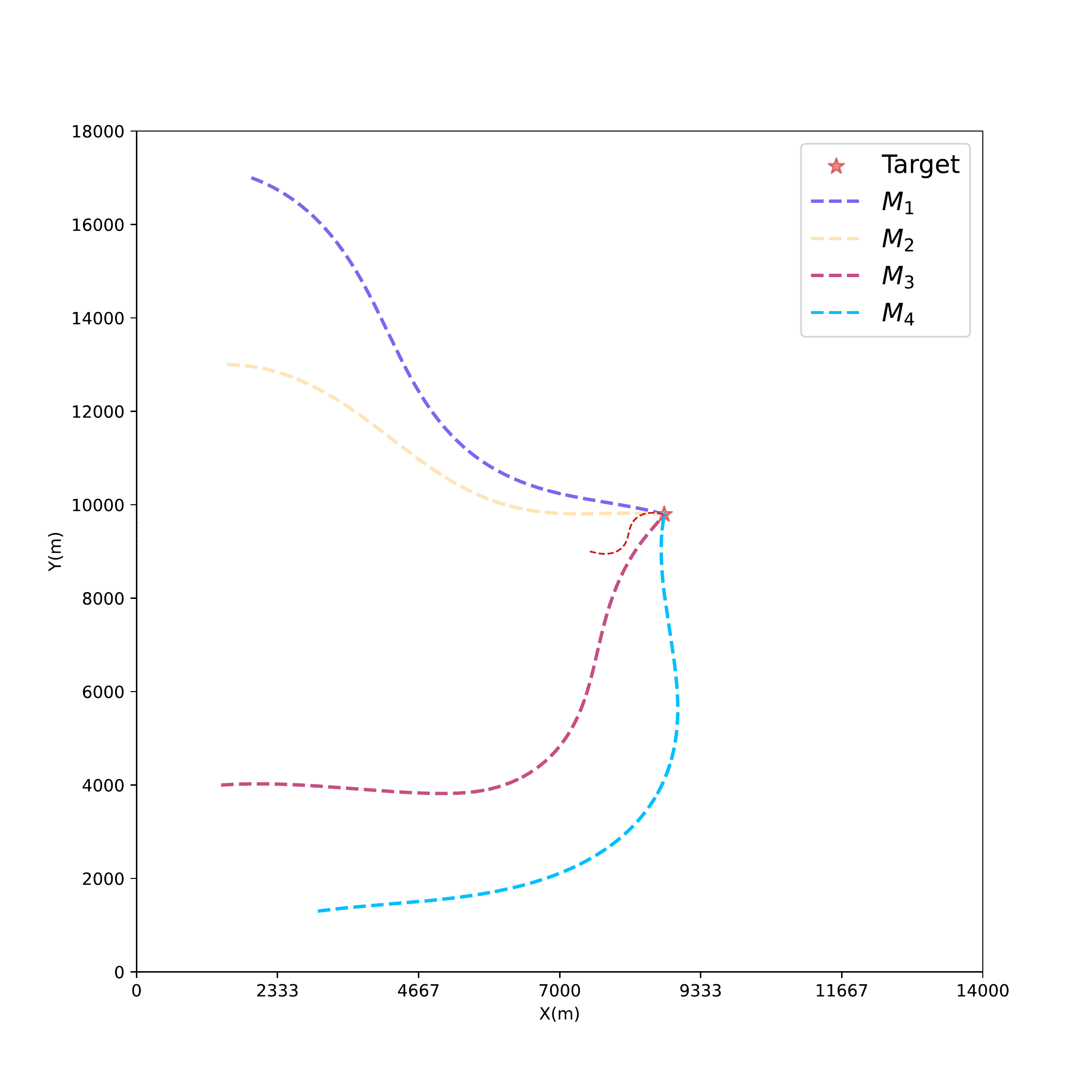}
		\caption{Trajectory in case 2}\label{fig:trajCase2}
	\end{figure}
	
	\begin{figure}[htb]
		\centering
		\begin{subfigure}[ht]{0.5\textwidth}
			\centering
			\includegraphics[width=\textwidth]{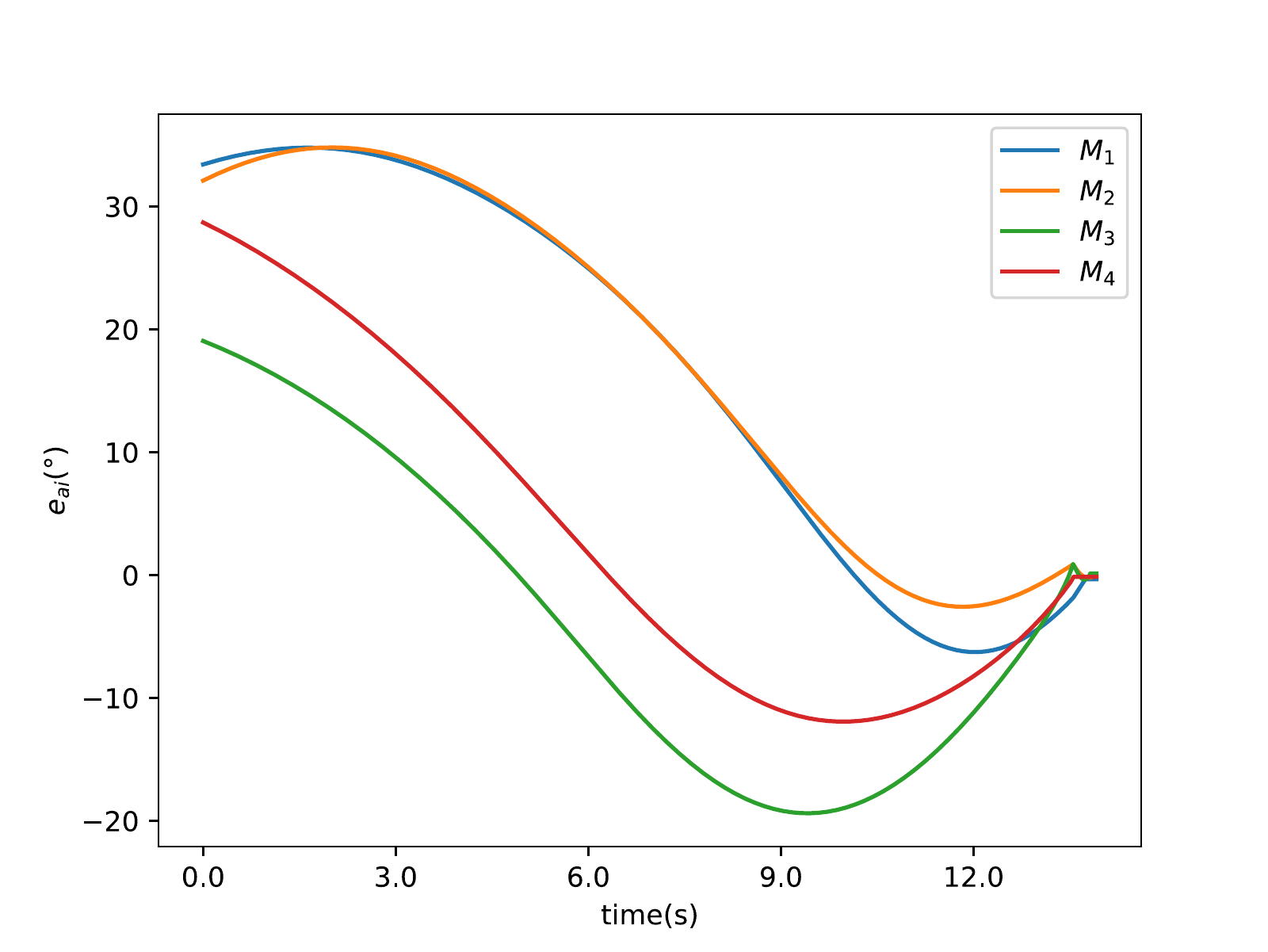}
			\caption{Consensus angle error profile}\label{fig:loserrorCase2}
		\end{subfigure}
		\begin{subfigure}[ht]{0.5\textwidth}
			\centering
			\includegraphics[width=\textwidth]{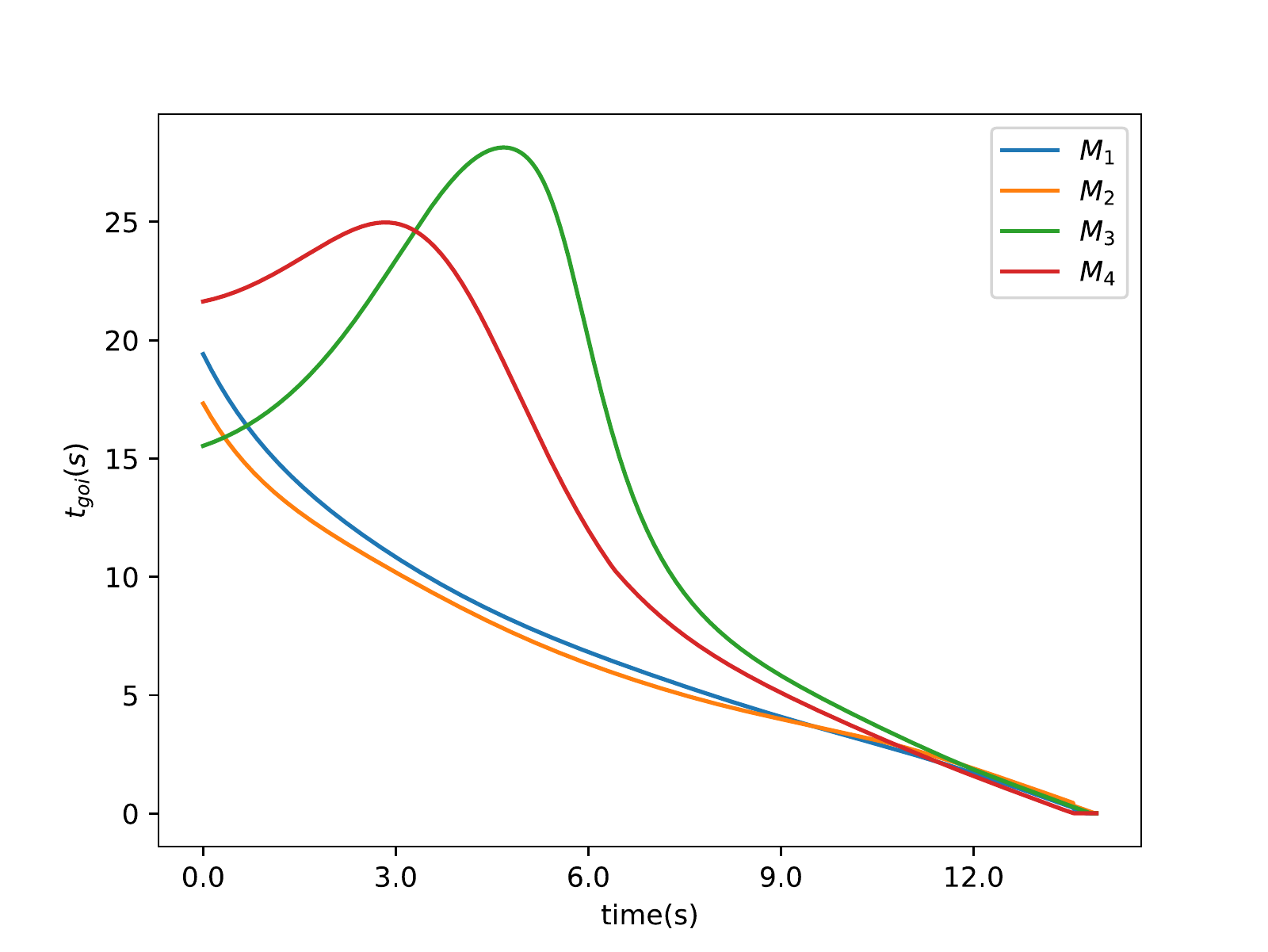}
			\caption{Time-to-goes profile}\label{fig:ttgerrorCase2}
		\end{subfigure}
		\caption{Flight data profiles in Case 2}\label{fig:profilesCase2}
	\end{figure}

	From \prettyref{tab:ResultForCase2} we can see that the consensus angle error is within one degree, which is sufficient for the accuracy requirement, and salvo attack is achieved with negligible consensus time error. 
	The result demonstrates the effectiveness of the proposed guidance method in intercepting the dynamic target. As far as the author knows, it is the first time achieving cooperative guidance against non-stationary target with intelligent control, which shows its extraordinary robustness against disturbance from non-stationary objectives.
	
	\begin{table}[h]
		\centering
		\caption{Result for Case 2.}
		\label{tab:ResultForCase2}
		\begin{tabular}{@{}lllll@{}}
			\hline\hline
			Index                         & $M_1$    & $M_2$   & $M_3$    & $M_4$     \\ \midrule
			$e_a^i$($\degree$) & -3.43\rm{E}-1 & -6.53\rm{E}-2 & -1.23\rm{E}-1 & -1.49\rm{E}-1 \\
			$e_t^i$(s)            & 6.00\rm{E}-2    & 1.85\rm{E}-2   & 2.10\rm{E}-1   & -4.55\rm{E}-1      \\
			ZEM(m)              & 1.83\rm{E}-3  & 6.38\rm{E}-3 & 2.49\rm{E}-2  & 9.28\rm{E}-1   \\ \hline\hline
		\end{tabular}
	\end{table}

	\subsection{Case 3: Monte-Carlo simulation}
	Monte-Carlo simulation has been extensively employed to examine the robustness of an algorithm under varying initial conditions, thus it is applied in this section. In the existing literature, the target is usually regarded as stationary as interception of a stationary target is more exclusive of unpredictable disturbance. In this case, five missiles are engaged, and each missile's position is randomly sampled from a uniform distribution, which is denoted by  $U(\cdot, \cdot)$. Specifically, for the $i^{th}$ missile, the x-coordinate of its position is $U(2000, 2600)$ and the y-coordinate is $U(11000, 13000) - 2000i$, which makes the missiles arranged in an orderly manner. The initial flight-path angles of all missiles are set to $0\degree$, with identical velocities of 600m/s and the same desired relative impact angles of $25\degree$. Additionally, the target's position is (10000m, 9000m).
	
	Simulations with randomly sampled conditions are conducted in 200 episodes. The diverse trajectories are depicted in \prettyref{fig:montecarloCase3}, and the statistical result after taking the absolute value is shown in \prettyref{tab:ResultForCase3}. From the result, we can see that the mean errors of impact angles are within $1\degree$, and the consensus error of impact time holds within $1\mathrm{s}$ most of the time. The result shows that for any initial state with limited error, the proposed scheme can always find the relative optimal solution.
	
	\begin{figure}[!htb]
		\centering
		\includegraphics[width=0.45\textwidth]{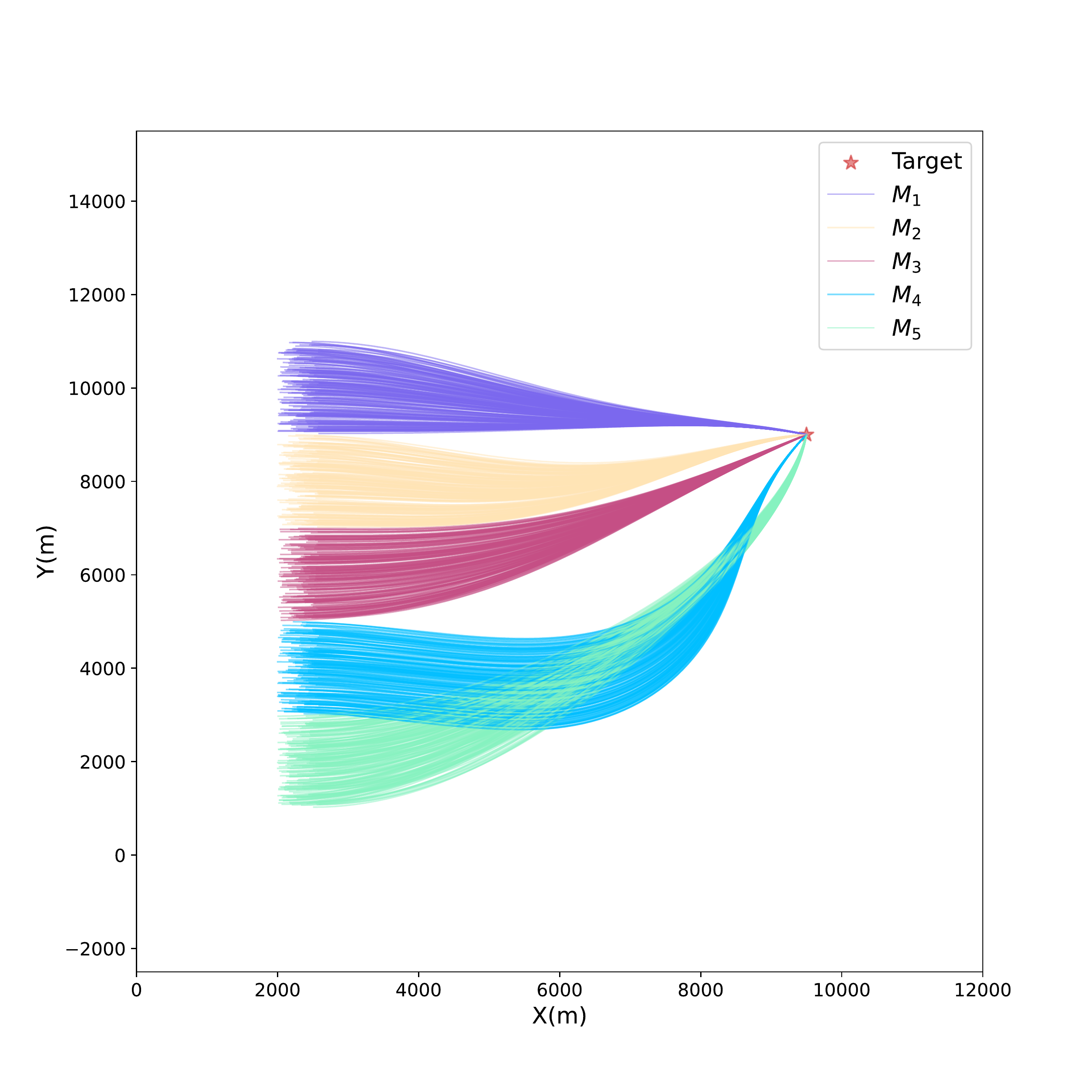}
		\caption{Diverse trajectories of the Monte-Carlo simulation}\label{fig:montecarloCase3}
	\end{figure}
	
	\begin{table}[!htb]
		\centering
		\caption{Result for Case 3.}
		\label{tab:ResultForCase3}
		\begin{tabular}{@{}p{0.13\columnwidth}llllll@{}}
			\hline\hline
			Index                                          &      & $M_1$   & $M_2$   & $M_3$   & $M_4$   & $M_5$   \\ \midrule
			\multirow{3}{*}{\makecell{$e_a^i$($\degree$)}} 
			& Mean & 4.50\rm{E}-1 & 8.20\rm{E}-1 & 7.10\rm{E}-1 & 2.20\rm{E}-1 & 9.30\rm{E}-1 \\ \cmidrule(lr){2-2}
			& Max  & 1.85\rm{E}-0 & 3.37\rm{E}-0 & 1.96\rm{E}-0 & 6.10\rm{E}-1 & 2.37\rm{E}-0 \\ \cmidrule(lr){2-2}
			& Min  & 4.56\rm{E}-3 & 6.40\rm{E}-3 & 7.01\rm{E}-4 & 4.58\rm{E}-4 & 6.46\rm{E}-3 \\ \midrule
			\multirow{3}{*}{\makecell{$e_t^i$(s)}}          
			& Mean & 6.10\rm{E}-1 & 5.50\rm{E}-1 & 5.30\rm{E}-1 & 4.50\rm{E}-1 & 5.50\rm{E}-1 \\ \cmidrule(lr){2-2}
			& Max  & 1.78\rm{E}-0 & 1.57\rm{E}-0 & 1.63\rm{E}-0 & 1.54\rm{E}-0 & 1.44\rm{E}-0 \\ \cmidrule(lr){2-2}
			& Min  & 1.50\rm{E}-2 & 1.00\rm{E}-2 & 1.78\rm{E}-15 & 5.00\rm{E}-3 & 1.78\rm{E}-15 \\ \midrule
			\multirow{3}{*}{\makecell{ZEM(m)}}              
			& Mean & 5.85\rm{E}-3 & 5.89\rm{E}-3 & 3.74\rm{E}-4 & 9.05\rm{E}-4 & 9.93\rm{E}-4 \\ \cmidrule(lr){2-2}
			& Max  & 1.03\rm{E}-2 & 1.07\rm{E}-2 & 7.77\rm{E}-4 & 2.53\rm{E}-3 & 3.39\rm{E}-3 \\ \cmidrule(lr){2-2}
			& Min  & 2.18\rm{E}-3 & 2.04\rm{E}-3 & 6.42\rm{E}-5 & 1.68\rm{E}-5 & 2.72\rm{E}-6 \\ \hline\hline
		\end{tabular}
	\end{table}
	
	\subsection{Optimization process analysis}
	\prettyref{fig:meanfitsAll} shows the learning curves in the three cases. The mean fitness in case 1 keeps moving upper and merges together at the final phase. From the curve of case 2, we can see that two of the missiles get ahead about 1000 scores, but finally back to meet with the other missiles. A similar phenomenon also appears in case 3. It can be inferred that the policies asymptotically evolved to the equilibrium state, and one reason is that the rescaled gradient prohibited the ever-increasing gap between individual groups, which is crucial for mutual improvement. If one group gets ahead too much, then the other groups may never chase up due to the interrelationship, which is to say that the improvement of the poorer-performed group is prohibited when more significant drops in the better-performed ones will occur.
	\prettyref{fig:lrAll} presents the adaptation profiles of learning rates applying the aforementioned technique. For case 1 and case 2, the learning rates start from high values and gradually converge to the minimal value, which corresponds with the quality of estimated gradients. However, due to the random initial conditions in case 3, the learning rates will not settle easily. 
	The extensive empirical result shows that without the learning rate adaptation, the fitness profiles will jitter in the end instead of converging to satisfactory ranges (regardless of the types of optimizer). Note that it is pretty common when training neural networks and may presumably have been caused by overfitting, according to related research in the field. Employing the simple adaptation technique contributes to distressing this deficiency.
	\begin{figure}[ht]
		\centering
		\begin{subfigure}[ht]{0.45\textwidth}
			\centering
			\includegraphics[width=\textwidth]{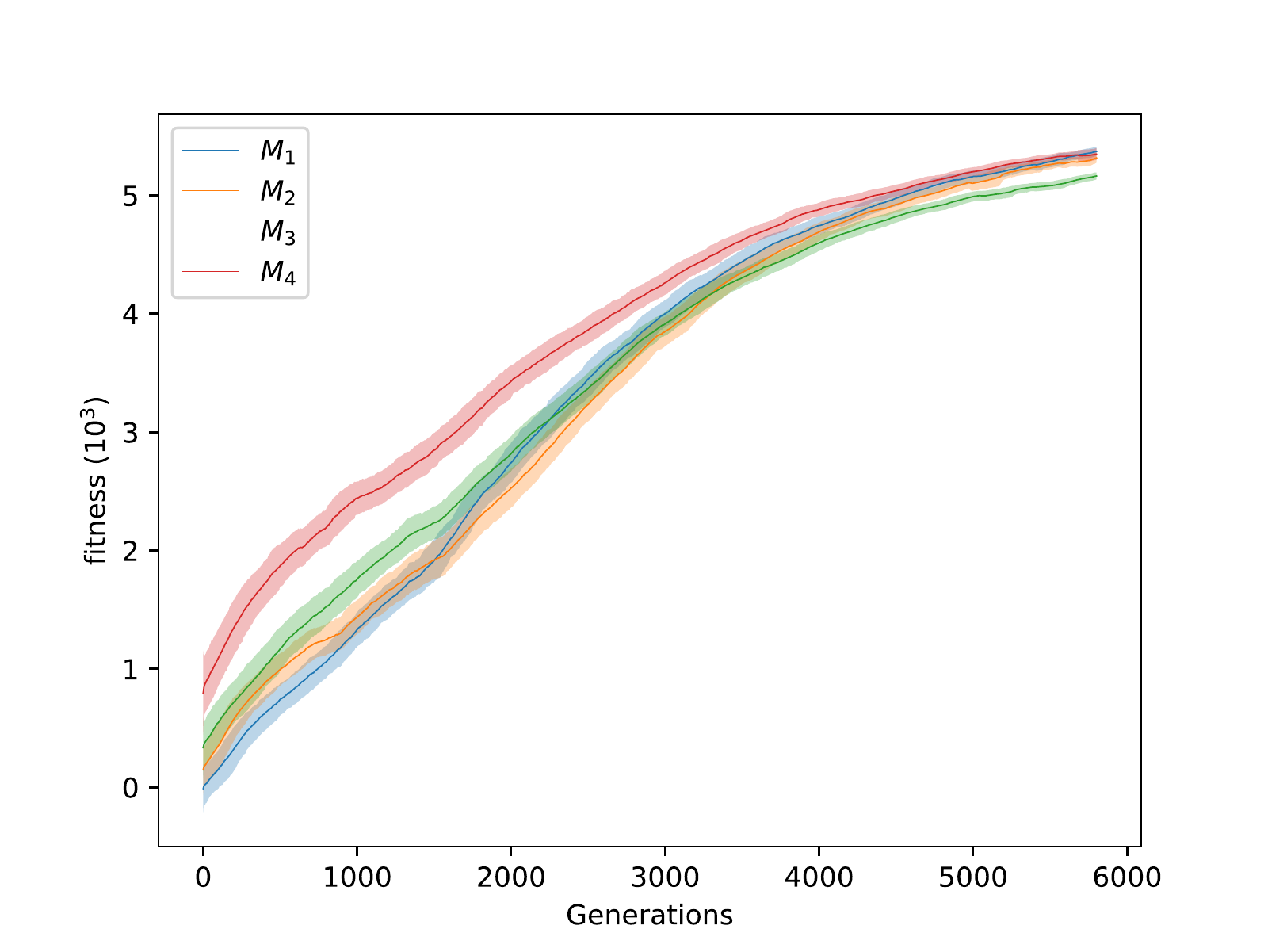}
			\caption{Case 1}\label{fig:meanfitsCase1}
		\end{subfigure}
		\begin{subfigure}[ht]{0.45\textwidth}
			\centering
			\includegraphics[width=\textwidth]{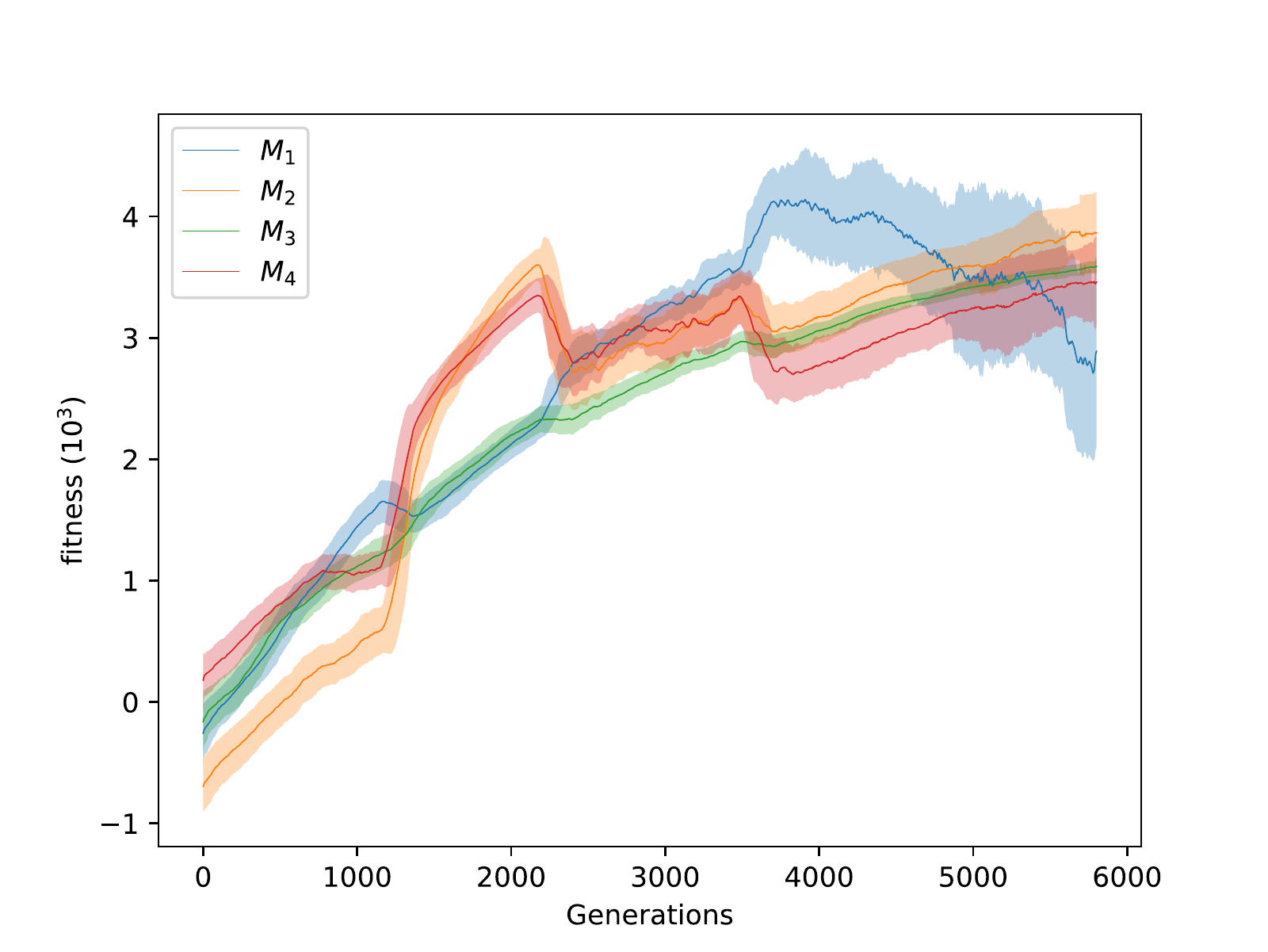}
			\caption{Case 2}\label{fig:meanfitsCase2}
		\end{subfigure}
		
		\begin{subfigure}[ht]{0.45\textwidth}
			\centering
			\includegraphics[width=\textwidth]{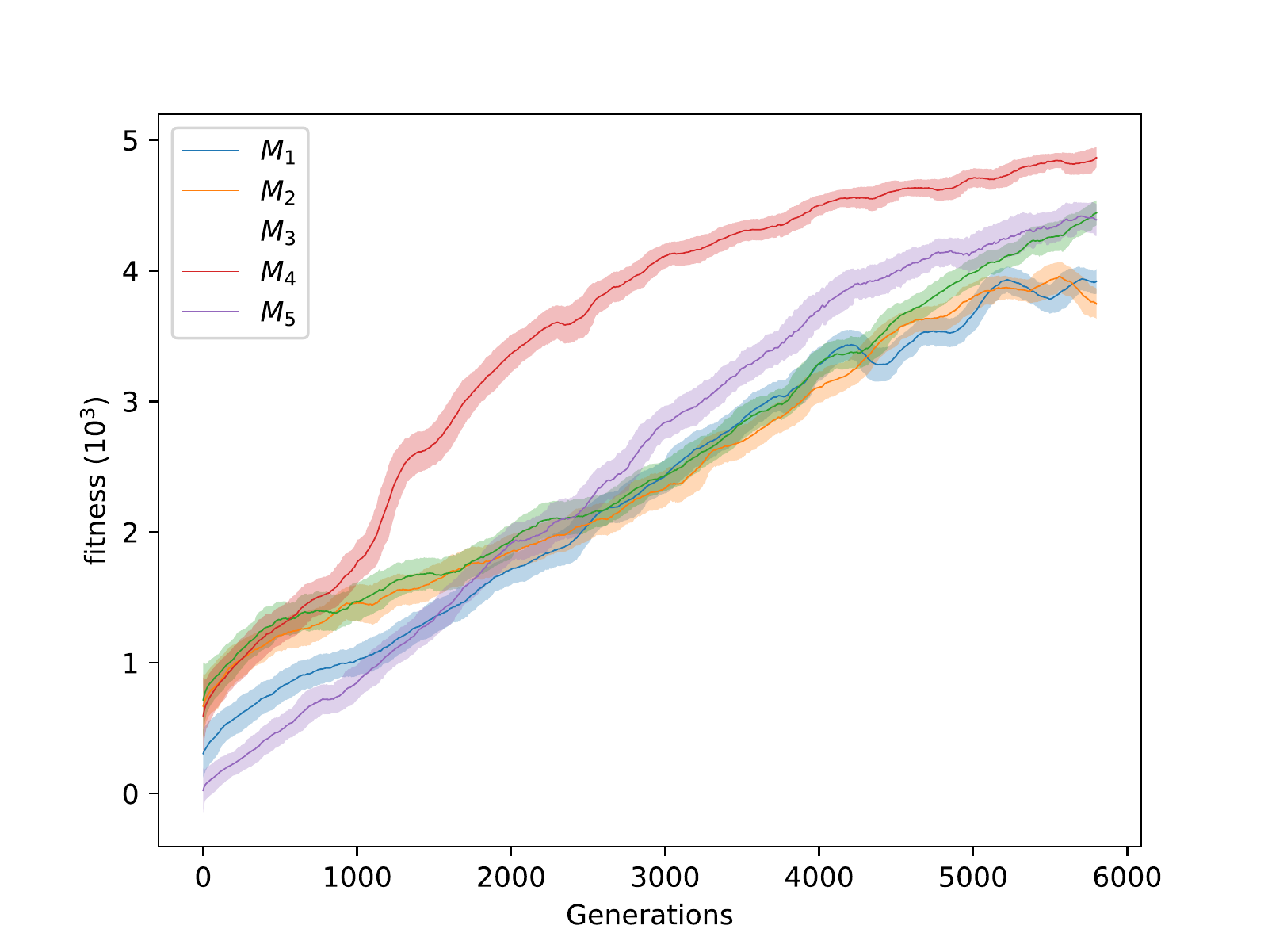}
			\caption{Case 3}\label{fig:meanfitsCase3}
		\end{subfigure}
		\caption{Mean fitness profiles over iterations of three cases}\label{fig:meanfitsAll}
	\end{figure}
	
	\begin{figure}[h]
		\centering
		\begin{subfigure}[ht]{0.45\textwidth}
			\centering
			\includegraphics[width=\textwidth]{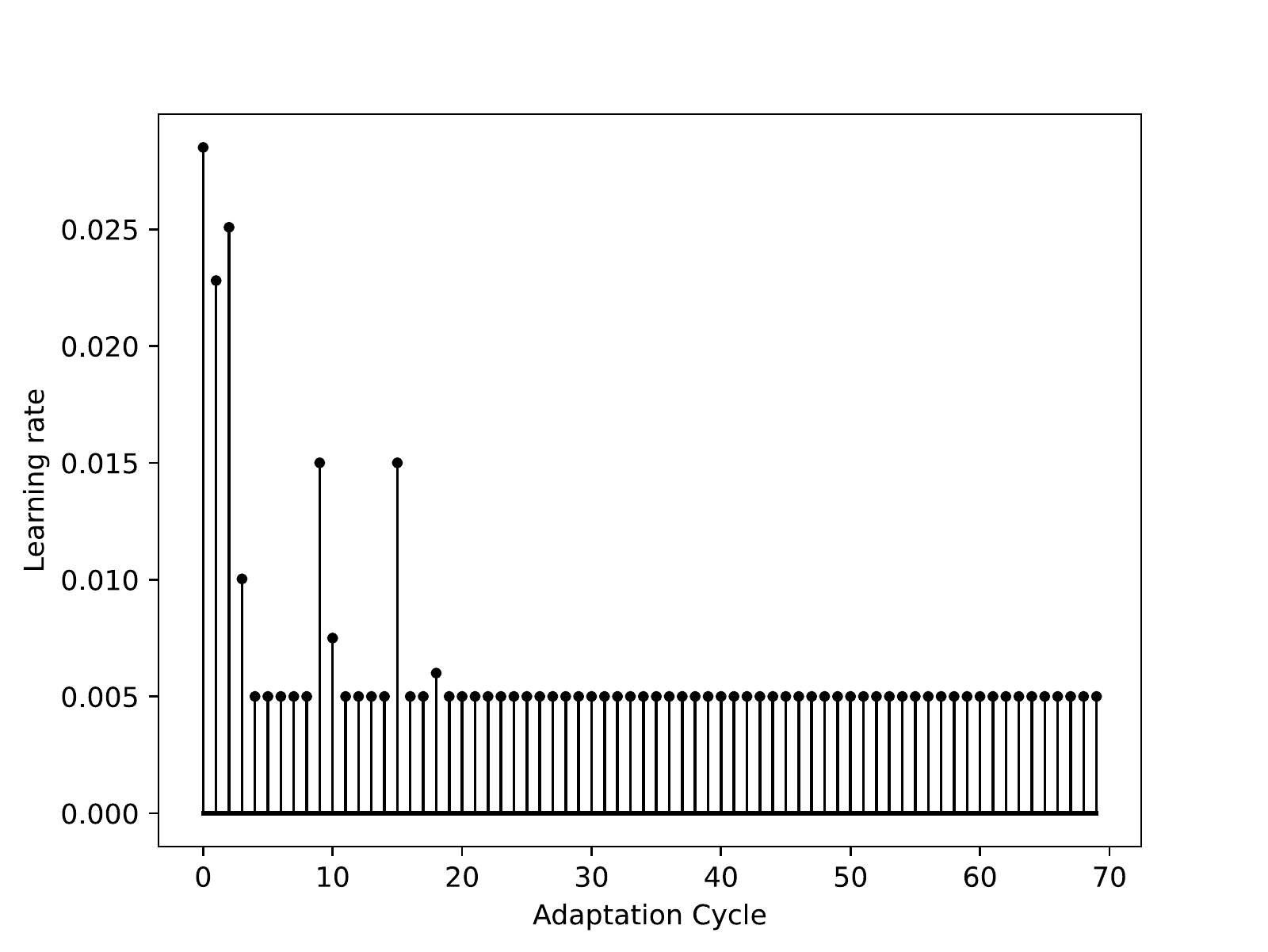}
			\caption{Case 1}\label{fig:lrCase1}
		\end{subfigure}
		\begin{subfigure}[ht]{0.45\textwidth}
			\centering
			\includegraphics[width=\textwidth]{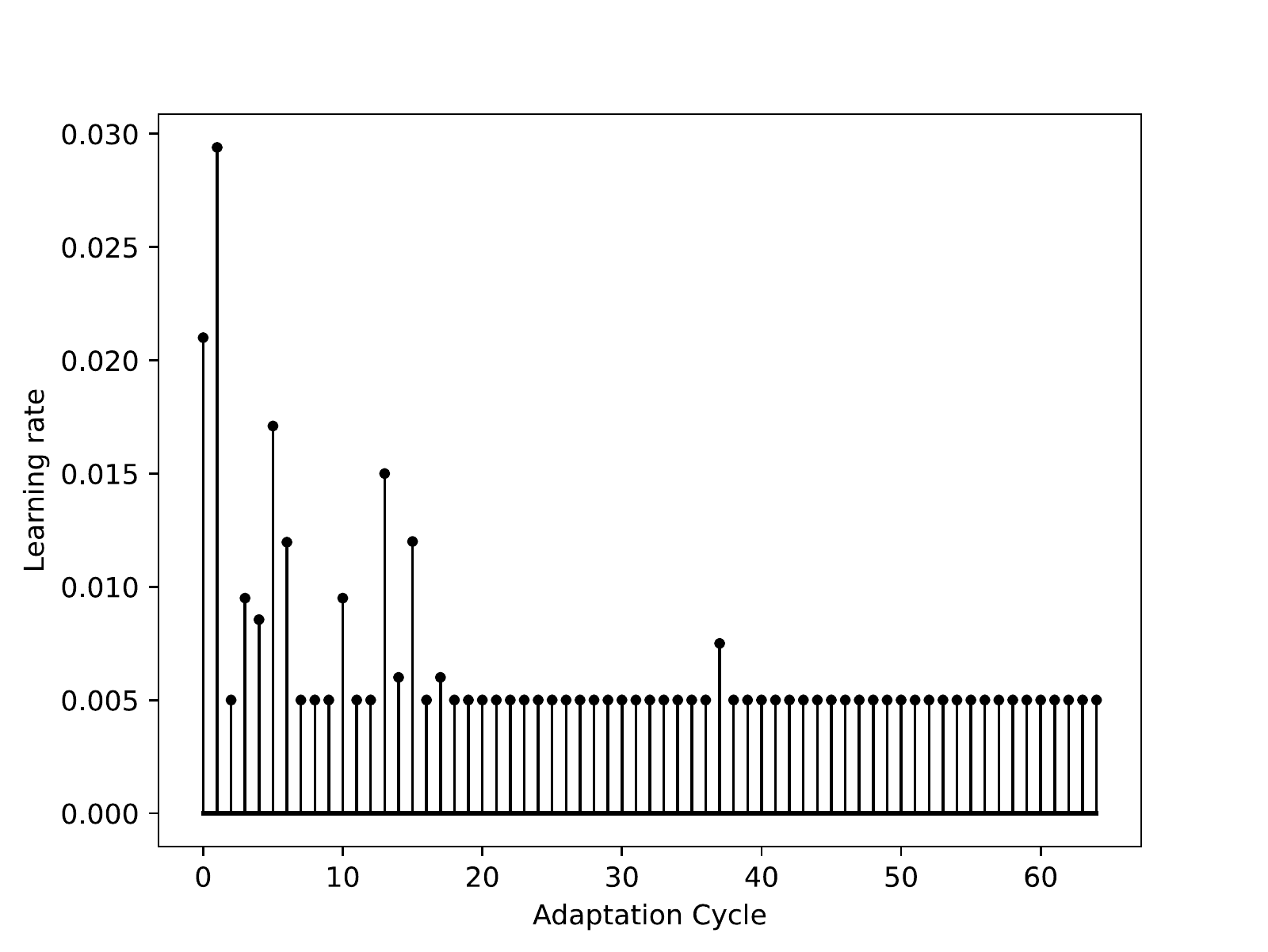}
			\caption{Case 2}\label{fig:lrCase2}
		\end{subfigure}
		
		\begin{subfigure}[ht]{0.45\textwidth}
			\centering
			\includegraphics[width=\textwidth]{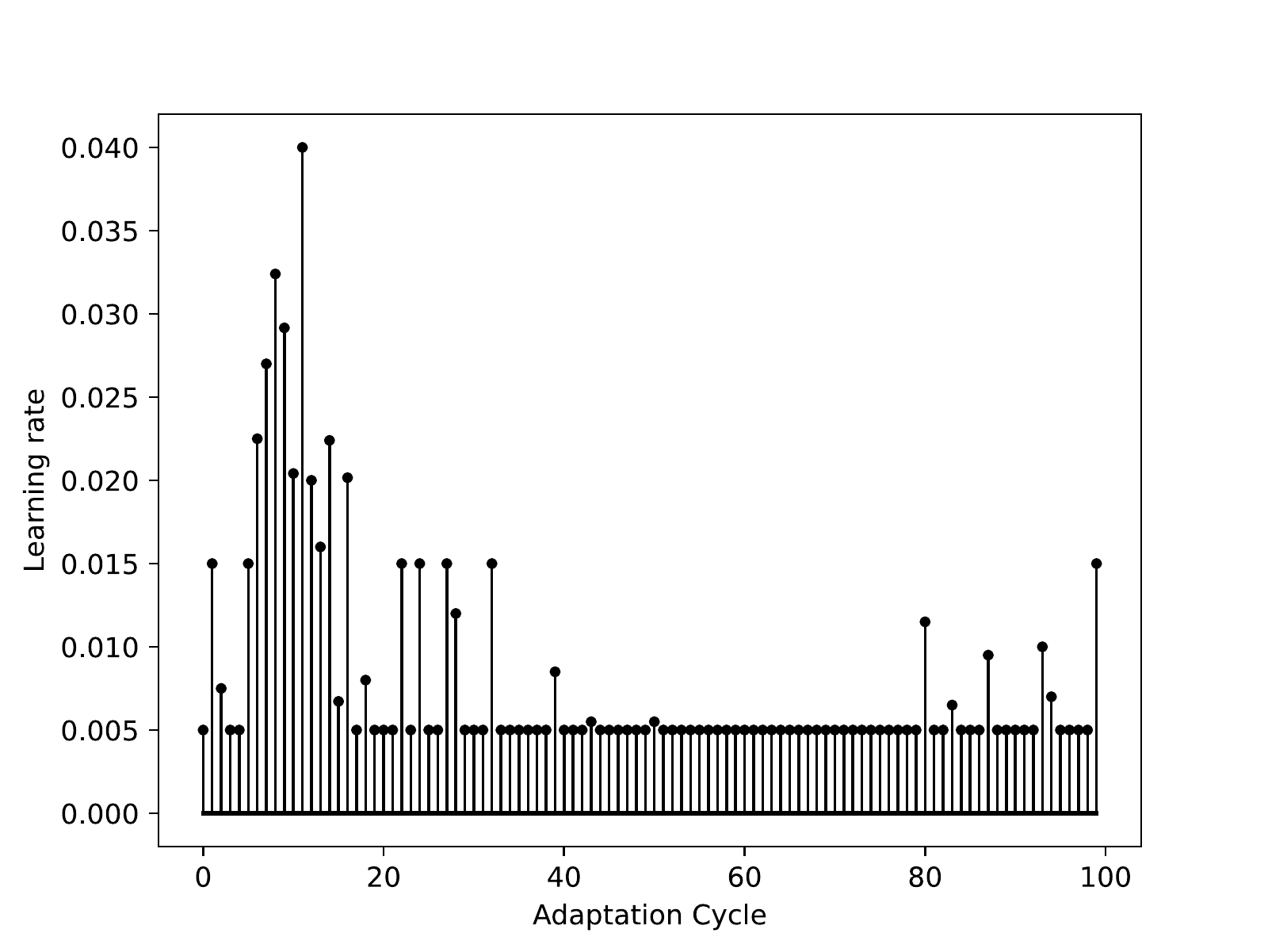}
			\caption{Case 3}\label{fig:lrCase3}
		\end{subfigure}
		\caption{Learning rate profiles over evaluation iterations of three cases}\label{fig:lrAll}
	\end{figure}
	
	\section{Conclusions}
	\label{sec:conclusions}
	
	In this paper, an improved co-evolutionary strategy NCES has been developed to solve the non-stationarity issue in multi-agent dynamic environments. The hybrid co-evolutionary cooperative guidance law(HCCGL) has been proposed to integrate with the improved strategy, and the neural network has been used to construct the consensus controller. To fully demonstrate its effectiveness in synchronizing impact time and angles, three experiments under different conditions have been carried out. Experiment on maneuvering target has been proven effective with satisfactory precision. The proposed method is shown to be robust and can be well scaled to solve the cooperative guidance problem for the multi-agent system, which is the first time an intelligent cooperative guidance law is applied to intercept a non-stationary target with time and angle constraints in the existing studies. 
	
	The proposed algorithm combines traditional control theories with intelligent algorithms, revealing the enormous potential in this field. It is always meaningful to explore the limits of modern control tasks. Despite the satisfactory results that have been acquired, this work still left space to be improved. Future works may include exploring the effectiveness of incremental guidance gain, or control strategies that tackle actuation failure and system uncertainty.
	
	%

	\clearpage
	\bibliographystyle{IEEEtran}
	\bibliography{mybibe_Coevolution}

\begin{thebibliography}{10}
\providecommand{\url}[1]{#1}
\csname url@samestyle\endcsname
\providecommand{\newblock}{\relax}
\providecommand{\bibinfo}[2]{#2}
\providecommand{\BIBentrySTDinterwordspacing}{\spaceskip=0pt\relax}
\providecommand{\BIBentryALTinterwordstretchfactor}{4}
\providecommand{\BIBentryALTinterwordspacing}{\spaceskip=\fontdimen2\font plus
\BIBentryALTinterwordstretchfactor\fontdimen3\font minus
  \fontdimen4\font\relax}
\providecommand{\BIBforeignlanguage}[2]{{%
\expandafter\ifx\csname l@#1\endcsname\relax
\typeout{** WARNING: IEEEtran.bst: No hyphenation pattern has been}%
\typeout{** loaded for the language `#1'. Using the pattern for}%
\typeout{** the default language instead.}%
\else
\language=\csname l@#1\endcsname
\fi
#2}}
\providecommand{\BIBdecl}{\relax}
\BIBdecl

\bibitem{jeon2010homing}
I.-S. Jeon, J.-I. Lee, and M.-J. Tahk, ``Homing guidance law for cooperative
  attack of multiple missiles,'' \emph{Journal of guidance, control, and
  dynamics}, vol.~33, no.~1, pp. 275--280, 2010.

\bibitem{ma2013guidance}
K.~Ma, H.~K. Khalil, and Y.~Yao, ``Guidance law implementation with performance
  recovery using an extended high-gain observer,'' \emph{Aerospace Science and
  Technology}, vol.~24, no.~1, pp. 177--186, 2013.

\bibitem{xiong2018hyperbolic}
S.~Xiong, M.~Wei, M.~Zhao, H.~Xiong, W.~Wang, and B.~Zhou, ``Hyperbolic tangent
  function weighted optimal intercept angle guidance law,'' \emph{Aerospace
  Science and Technology}, vol.~78, pp. 604--619, 2018.

\bibitem{liRobustCooperativeGuidance2019}
Z.~Li and Z.~Ding, ``Robust {{Cooperative Guidance Law}} for {{Simultaneous
  Arrival}},'' \emph{IEEE Transactions on Control Systems Technology}, vol.~27,
  no.~3, pp. 1360--1367, May 2019.

\bibitem{he2021computational}
S.~He, H.-S. Shin, and A.~Tsourdos, ``Computational missile guidance: a deep
  reinforcement learning approach,'' \emph{Journal of Aerospace Information
  Systems}, vol.~18, no.~8, pp. 571--582, 2021.

\bibitem{ratnoo2008impact}
A.~Ratnoo and D.~Ghose, ``Impact angle constrained interception of stationary
  targets,'' \emph{Journal of Guidance, Control, and Dynamics}, vol.~31, no.~6,
  pp. 1817--1822, 2008.

\bibitem{jeon2006impact}
I.-S. Jeon, J.-I. Lee, and M.-J. Tahk, ``Impact-time-control guidance law for
  anti-ship missiles,'' \emph{IEEE Transactions on control systems technology},
  vol.~14, no.~2, pp. 260--266, 2006.

\bibitem{gaudet2020reinforcement}
B.~Gaudet, R.~Furfaro, and R.~Linares, ``Reinforcement learning for angle-only
  intercept guidance of maneuvering targets,'' \emph{Aerospace Science and
  Technology}, vol.~99, p. 105746, 2020.

\bibitem{la2014multirobot}
H.~M. La, R.~Lim, and W.~Sheng, ``Multirobot cooperative learning for predator
  avoidance,'' \emph{IEEE Transactions on Control Systems Technology}, vol.~23,
  no.~1, pp. 52--63, 2014.

\bibitem{dong2021composite}
H.~Dong and X.~Zhao, ``Composite experience replay-based deep reinforcement
  learning with application in wind farm control,'' \emph{IEEE Transactions on
  Control Systems Technology}, vol.~30, no.~3, pp. 1281--1295, 2021.

\bibitem{kong2020maneuver}
W.~Kong, D.~Zhou, Z.~Yang, K.~Zhang, and L.~Zeng, ``Maneuver strategy
  generation of ucav for within visual range air combat based on multi-agent
  reinforcement learning and target position prediction,'' \emph{Applied
  Sciences}, vol.~10, no.~15, p. 5198, 2020.

\bibitem{albaba2021driver}
B.~M. Albaba and Y.~Yildiz, ``Driver modeling through deep reinforcement
  learning and behavioral game theory,'' \emph{IEEE Transactions on Control
  Systems Technology}, vol.~30, no.~2, pp. 885--892, 2021.

\bibitem{chen2021communication}
T.~Chen, K.~Zhang, G.~B. Giannakis, and T.~Ba{\c{s}}ar,
  ``Communication-efficient policy gradient methods for distributed
  reinforcement learning,'' \emph{IEEE Transactions on Control of Network
  Systems}, vol.~9, no.~2, pp. 917--929, 2021.

\bibitem{dong2020reinforcement}
H.~Dong, X.~Zhao, and H.~Yang, ``Reinforcement learning-based approximate
  optimal control for attitude reorientation under state constraints,''
  \emph{IEEE Transactions on Control Systems Technology}, vol.~29, no.~4, pp.
  1664--1673, 2020.

\bibitem{Liangchen2021Metalearning}
L.~Chen, W.~{Wei hong}, and L.~Chao, ``Deep reinforcement meta-learning
  guidance with impact angle constraint,'' \emph{Journal of Astronautics},
  2021.

\bibitem{zhang2016data}
H.~Zhang, H.~Jiang, Y.~Luo, and G.~Xiao, ``Data-driven optimal consensus
  control for discrete-time multi-agent systems with unknown dynamics using
  reinforcement learning method,'' \emph{IEEE Transactions on Industrial
  Electronics}, vol.~64, no.~5, pp. 4091--4100, 2016.

\bibitem{nguyen2020deep}
T.~T. Nguyen, N.~D. Nguyen, and S.~Nahavandi, ``Deep reinforcement learning for
  multiagent systems: A review of challenges, solutions, and applications,''
  \emph{IEEE transactions on cybernetics}, vol.~50, no.~9, pp. 3826--3839,
  2020.

\bibitem{brockhoff2010mirrored}
D.~Brockhoff, A.~Auger, N.~Hansen, D.~V. Arnold, and T.~Hohm, ``Mirrored
  sampling and sequential selection for evolution strategies,'' in
  \emph{International Conference on Parallel Problem Solving from
  Nature}.\hskip 1em plus 0.5em minus 0.4em\relax Springer, 2010, pp. 11--21.

\bibitem{salimans2017evolution}
T.~Salimans, J.~Ho, X.~Chen, S.~Sidor, and I.~Sutskever, ``Evolution strategies
  as a scalable alternative to reinforcement learning,'' \emph{arXiv preprint
  arXiv:1703.03864}, 2017.

\bibitem{wierstra2014natural}
D.~Wierstra, T.~Schaul, T.~Glasmachers, Y.~Sun, J.~Peters, and J.~Schmidhuber,
  ``Natural evolution strategies,'' \emph{The Journal of Machine Learning
  Research}, vol.~15, no.~1, pp. 949--980, 2014.

\bibitem{xu2017environment}
B.~Xu, Y.~Zhang, D.~Gong, Y.~Guo, and M.~Rong, ``Environment sensitivity-based
  cooperative co-evolutionary algorithms for dynamic multi-objective
  optimization,'' \emph{IEEE/ACM transactions on computational biology and
  bioinformatics}, vol.~15, no.~6, pp. 1877--1890, 2017.

\bibitem{qu2013improved}
H.~Qu, K.~Xing, and T.~Alexander, ``An improved genetic algorithm with
  co-evolutionary strategy for global path planning of multiple mobile
  robots,'' \emph{Neurocomputing}, vol. 120, pp. 509--517, 2013.

\bibitem{wang2022instance}
Z.~Wang, C.~Chen, and D.~Dong, ``Instance weighted incremental evolution
  strategies for reinforcement learning in dynamic environments,'' \emph{IEEE
  Transactions on Neural Networks and Learning Systems}, 2022.

\bibitem{larranaga2001estimation}
P.~Larra{\~n}aga and J.~A. Lozano, \emph{Estimation of distribution algorithms:
  A new tool for evolutionary computation}.\hskip 1em plus 0.5em minus
  0.4em\relax Springer Science \& Business Media, 2001, vol.~2.

\bibitem{karshenas2013regularized}
H.~Karshenas, R.~Santana, C.~Bielza, and P.~Larra{\~n}aga, ``Regularized
  continuous estimation of distribution algorithms,'' \emph{Applied Soft
  Computing}, vol.~13, no.~5, pp. 2412--2432, 2013.

\bibitem{omar2011multiobjective}
H.~M. Omar and M.~Abido, ``Multiobjective evolutionary algorithm for designing
  fuzzy-based missile guidance laws,'' \emph{Journal of Aerospace Engineering},
  vol.~24, no.~1, pp. 89--94, 2011.

\bibitem{maheswaranathan2019guided}
N.~Maheswaranathan, L.~Metz, G.~Tucker, D.~Choi, and J.~Sohl-Dickstein,
  ``Guided evolutionary strategies: Augmenting random search with surrogate
  gradients,'' in \emph{International Conference on Machine Learning}.\hskip
  1em plus 0.5em minus 0.4em\relax PMLR, 2019, pp. 4264--4273.

\bibitem{del2019bio}
J.~Del~Ser, E.~Osaba, D.~Molina, X.-S. Yang, S.~Salcedo-Sanz, D.~Camacho,
  S.~Das, P.~N. Suganthan, C.~A.~C. Coello, and F.~Herrera, ``Bio-inspired
  computation: Where we stand and what's next,'' \emph{Swarm and Evolutionary
  Computation}, vol.~48, pp. 220--250, 2019.

\bibitem{gray2018multiagent}
R.~Gray, A.~Franci, V.~Srivastava, and N.~E. Leonard, ``Multiagent
  decision-making dynamics inspired by honeybees,'' \emph{IEEE Transactions on
  Control of Network Systems}, vol.~5, no.~2, pp. 793--806, 2018.

\bibitem{gaudetReinforcementLearningAngleOnly2020}
\BIBentryALTinterwordspacing
B.~Gaudet, R.~Furfaro, and R.~Linares, ``Reinforcement {{Learning}} for
  {{Angle-Only Intercept Guidance}} of {{Maneuvering Targets}},'' vol.~99, p.
  105746. [Online]. Available: \url{http://arxiv.org/abs/1906.02113}
\BIBentrySTDinterwordspacing

\bibitem{han2002state}
D.~Han and S.~Balakrishnan, ``State-constrained agile missile control with
  adaptive-critic-based neural networks,'' \emph{IEEE Transactions on Control
  Systems Technology}, vol.~10, no.~4, pp. 481--489, 2002.

\bibitem{li2020adaptive}
H.~Li, Y.~Wu, and M.~Chen, ``Adaptive fault-tolerant tracking control for
  discrete-time multiagent systems via reinforcement learning algorithm,''
  \emph{IEEE Transactions on Cybernetics}, vol.~51, no.~3, pp. 1163--1174,
  2020.

\bibitem{thananjeyan2021recovery}
B.~Thananjeyan, A.~Balakrishna, S.~Nair, M.~Luo, K.~Srinivasan, M.~Hwang, J.~E.
  Gonzalez, J.~Ibarz, C.~Finn, and K.~Goldberg, ``Recovery rl: Safe
  reinforcement learning with learned recovery zones,'' \emph{IEEE Robotics and
  Automation Letters}, vol.~6, no.~3, pp. 4915--4922, 2021.

\bibitem{sonQTRANLearningFactorize2019}
K.~Son, D.~Kim, W.~J. Kang, D.~E. Hostallero, and Y.~Yi, ``Qtran: Learning to
  factorize with transformation for cooperative multi-agent reinforcement
  learning,'' in \emph{International conference on machine learning}.\hskip 1em
  plus 0.5em minus 0.4em\relax PMLR, 2019, pp. 5887--5896.

\bibitem{gong2015distributed}
Y.-J. Gong, W.-N. Chen, Z.-H. Zhan, J.~Zhang, Y.~Li, Q.~Zhang, and J.-J. Li,
  ``Distributed evolutionary algorithms and their models: A survey of the
  state-of-the-art,'' \emph{Applied Soft Computing}, vol.~34, pp. 286--300,
  2015.

\bibitem{mendiburu2005parallel}
A.~Mendiburu, J.~A. Lozano, and J.~Miguel-Alonso, ``Parallel implementation of
  edas based on probabilistic graphical models,'' \emph{IEEE Transactions on
  Evolutionary Computation}, vol.~9, no.~4, pp. 406--423, 2005.

\bibitem{trennMultilayerPerceptronsApproximation2008}
S.~Trenn, ``Multilayer perceptrons: Approximation order and necessary number of
  hidden units,'' \emph{IEEE transactions on neural networks}, vol.~19, no.~5,
  pp. 836--844, 2008.

\bibitem{lyu2019multiple}
T.~Lyu, Y.~Guo, C.~Li, G.~Ma, and H.~Zhang, ``Multiple missiles cooperative
  guidance with simultaneous attack requirement under directed topologies,''
  \emph{Aerospace Science and Technology}, vol.~89, pp. 100--110, 2019.

\end{thebibliography}

\end{document}